\documentclass{article}

    \usepackage[preprint]{neurips_2025}

\usepackage[utf8]{inputenc} 
\usepackage[T1]{fontenc}    
\usepackage{hyperref}       
\usepackage{url}            
\usepackage{booktabs}       
\usepackage{amsfonts}       
\usepackage{nicefrac}       
\usepackage{microtype}      
\usepackage{xcolor}         
\usepackage{enumitem}

\usepackage{amsmath,amsfonts}
\usepackage{algorithmic}
\usepackage{algorithm}
\usepackage{array}
\usepackage{textcomp}
\usepackage{stfloats}
\usepackage{verbatim}
\usepackage{graphicx}

\usepackage{hyperref}       
\usepackage{amsthm, amsmath}
\usepackage{multirow}

\theoremstyle{plain}
\newtheorem{theorem}{{T{\footnotesize HEOREM}}}[section]
\newtheorem{lemma}[theorem]{{L{\footnotesize EMMA}}}
\newtheorem{proposition}[theorem]{{P{\footnotesize ROPOSITION}}}

\theoremstyle{definition}
\newtheorem{definition}[theorem]{{D{\footnotesize EFINITION}}}

\newtheorem{remark}[theorem]{{R{\footnotesize EMARK}}}

\newcommand{\aA}{\mathcal{A}}

\newcommand{\lL}{\mathcal{L}}

\newcommand{\nN}{\mathcal{N}}

\newcommand{\sS}{\mathcal{S}}

\newcommand{\NN}{\mathbb{N}}

\newcommand{\RR}{\mathbb{R}}

\newcommand{\op}{{\operatorname{op}}}
\newcommand{\diag}{{\operatorname{diag}}}
\newcommand{\Hom}{{\operatorname{Hom}}}

\usepackage{wrapfig} 

\title{Lipschitz-Driven Noise Robustness in VQ-AE for High-Frequency Texture Repair in ID-Specific Talking Heads}

\author{%
    Jian Yang$^2$ \and 
    Xukun Wang$^2$ \and 
    Wentao Wang$^3$ \and 
    Guoming Li$^2$ \and
    Qihang Fang$^4$ \and
    Ruihong Yuan$^2$ \and
    Tianyang Wang$^3$ \and
    Xiaomei Zhang$^4$ \and
    Yeying Jin$^5$ \and
    {Zhaoxin Fan$^1$}\thanks{Corresponding author. Email: fanzhaoxinruc@gmail.com} \and
}

\begin{document}
    \maketitle
    \vspace{-3em} 
    \begin{center}
        $^1$Beihang University~~$^2$Psyche AI Inc.~~$^3$The University of Alabama at Birmingham~~$^4$CASIA ~~$^5$Tencent \\
    \end{center}

\maketitle

\begin{abstract}
Audio-driven IDentity-specific Talking Head Generation (ID-specific THG) has shown increasing promise for applications in filmmaking and virtual reality. 
Existing approaches are generally constructed as end-to-end paradigms, and have achieved significant progress. However, they often struggle to capture high-frequency textures due to limited model capacity. 
To address these limitations, we adopt a simple yet efficient post-processing framework---unlike previous studies that focus solely on end-to-end training---guided by our theoretical insights. 
Specifically, leveraging the \textit{Lipschitz Continuity Theory} of neural networks, we prove a crucial noise tolerance property for the Vector Quantized AutoEncoder (VQ-AE), and establish the existence of a Noise Robustness Upper Bound (NRoUB). 
This insight reveals that we can efficiently obtain an identity-specific denoiser by training an identity-specific neural discrete representation, without requiring an extra network. 
Based on this theoretical foundation, we propose a plug-and-play Space-Optimized VQ-AE (SOVQAE) with enhanced NRoUB to achieve temporally-consistent denoising. 
For practical deployment, we further introduce a cascade pipeline combining a pretrained Wav2Lip model with SOVQAE to perform ID-specific THG.
Our experiments demonstrate that this pipeline achieves \textit{state-of-the-art} performance in video quality and robustness for out-of-distribution lip synchronization, surpassing existing identity-specific THG methods. 
In addition, the pipeline requires only a couple of consumer GPU hours and runs in real time, which is both efficient and practical for industry applications.
\end{abstract}    
\section{Introduction}
\label{sec:intro}

The generation of photo-realistic, speech-driven, IDentity-specific Talking Head Generation (ID-specific THG) holds significant potential across diverse domains, including filmmaking \cite{kim2018deep}, virtual reality \cite{morishima1998real}, and digital avatar creation \cite{thies2020neural}. The goal of this task is to synthesize ID-specific talking head videos that achieve precise lip synchronization while preserving fine-grained details such as hair, facial wrinkles, moles, eyelashes, and the intricate contours of the lips. These high-frequency details are critical for enhancing realism in synthesized videos.

Existing NeRF-based approaches \cite{peng2023synctalk,guo2021ad,li2023efficient,ye2023geneface,ye2023geneface++} achieve impressive advances in preserving high-fidelity identities. However, these methods typically rely on multilayer perceptrons, which are inherently limited in capturing high-frequency components due to the well-documented \textit{spectral bias} phenomenon: this bias, as stated in \cite{rahaman2019spectral,basri2020frequency,yuce2022structured}, arises from the pathological eigenvalue distribution of the neural tangent kernel \cite{jacot2018neural,tancik2020fourier}, where most eigenvalues are small, limiting convergence on high-frequency components. 
On the other hand, blind face restoration methods~\cite{wang2023dr2,zhou2022towards,zhu2022blind}, common-used post-processing techniques, are capable of enhancing high-frequency textures, but often suffer from temporal inconsistencies such as texture flickering. While fine-tuning these models on ID-specific data can reduce flickering, it may degrade generalization and incurs additional inefficiencies. 

To address these challenges, we propose a simple yet effective post-processing method for ID-specific THG, grounded in rigorous theoretical insights. 
By analyzing the forward dynamics of VQ-AE {\footnote{We use the term VQ-AE to distinguish it from VQGAN \cite{esser2021taming}, which combines a VQ-AE with a transformer network and emphasizes image synthesis. Here, we focus on the auto-encoder nature of VQ-AE.}}, we uncover the intrinsic noise robustness of VQ-AE and formally derive its Noise Robustness Upper Bound (NRoUB) under the \textit{Lipschitz Continuity} framework \cite{balan2018lipschitz, zou2019lipschitz} (see Sec.~\ref{sec:denosing_modeling}). 
This theoretical result implies that training an ID-specific VQ-AE yields an efficient denoiser, where the denoising process remains lossless as long as the noise perturbation stays below the NRoUB. 
Building on this insight, we introduce a space regularization loss to further improve the NRoUB, resulting in a variant we call Space-Optimized VQ-AE (SOVQAE), which enables temporally consistent denoising (see Fig.~\ref{fig:nois_robustness}).


\begin{wrapfigure}{r}{0.6\textwidth}
\vspace{-0.2cm}
  \centering
  \includegraphics[width=\linewidth]{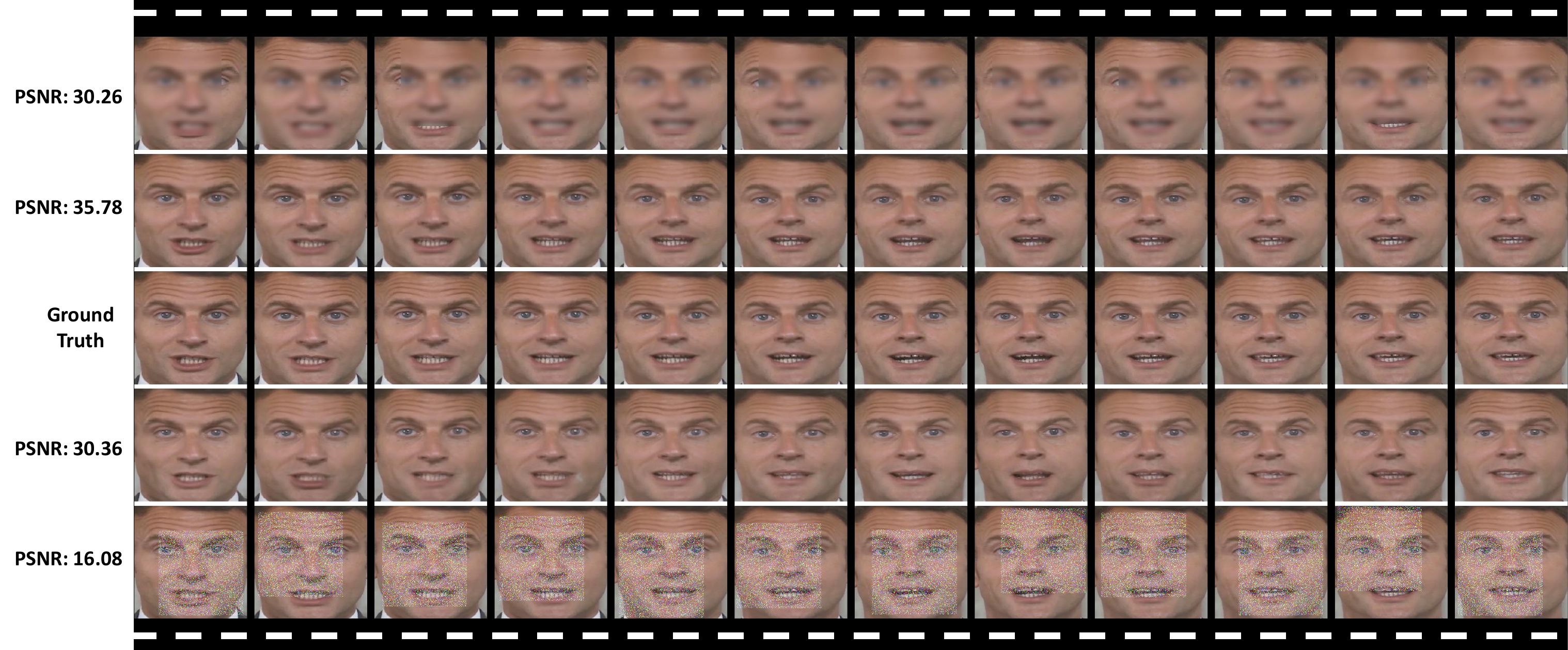}
  \vspace{-0.5cm}
  \caption{Display of the noise robustness performance in \textit{Macron} subject. Please zoom in for better visualization. From top to bottom, there are \textit{Gaussian blurred} video, denoised video by SOVQAE, ground truth video, denoised video and \textit{Gaussian noised} video. We pick 12 continuous frames clip to demonstrate the temporal-consistency.}
  \label{fig:nois_robustness}
  \vspace{-0.2cm}
\end{wrapfigure}

To translate this theoretical foundation into a practical THG pipeline, we design a cascade architecture combining a pretrained Wav2Lip model~\cite{prajwal2020lip}---a lightweight and effective lip synchronization model---with an ID-specific SOVQAE (detailed in Sec.~\ref{sec:method}). 
Empirical results show that our pipeline achieves new \textit{state-of-the-art} performance in both video quality and out-of-distribution lip synchronization accuracy, outperforming existing ID-specific THG methods. 
We attribute this performance to the robust audio-lip synchronization provided by the pretrained model and the high-frequency texture restoration enabled by SOVQAE. 
Moreover, with \textit{half-precision} training and inference, our pipeline runs at 30 FPS on a consumer-grade GPU, demonstrating its efficiency and practicality for industrial deployment.
In summary, our main contributions are as follows:

\begin{itemize}[leftmargin=20pt,parsep=2pt,itemsep=1pt,topsep=0pt]
    \item \textbf{Theoretical Contribution:} We establish the noise robustness of VQ-AE through the lens of Lipschitz continuity theory and derive a formal upper bound, termed the Noise Robustness Upper Bound (NRoUB).
    
    \item \textbf{Methodological Contribution:} We propose SOVQAE, which leverages the derived NRoUB to enable temporally consistent post-processing. Additionally, we introduce an efficient cascade pipeline for identity-specific talking-head generation, combining a pre-trained Wav2Lip model with SOVQAE. The pipeline is simple to train, requiring only a few consumer GPU hours, and operates in real-time inference.

    \item \textbf{Empirical Contribution:} Extensive experiments show that our cascade pipeline achieves state-of-the-art performance in video quality and robustness for out-of-distribution lip synchronization, highlighting its practical utility for real-world deployment.
\end{itemize}
\section{Related Work}
\paragraph{Audio-Driven Talking Heads Generation.}
As delineated in \cite{prajwal2020lip,shen2023difftalk,zhang2023dinet,wang2023seeing,zhong2023identity,wang2023lipformer,stypulkowski2024diffused,ma2023dreamtalk}, \textbf{one-shot} methods harness a plethora of talking video clips to capture the generalizability across multiple faces, thereby attaining commendable lip-synchronization performance. 
Despite these advancements, these methods continue to grapple with maintaining identity consistency across video frames. To counter this limitation, \textbf{ID-specific} NeRF-based methods \cite{peng2023synctalk,li2023efficient,ye2023geneface,guo2021ad,ye2023geneface++,tang2022real} are proposed, focusing exclusively on identity-specific data for training. Although this approach mitigates issues of identity instability, the constrained volume of training data hampers their cross-audio synchronization capabilities. To this end, GeneFace \cite{ye2023geneface} and SyncTalk \cite{peng2023synctalk} enhance their audio generalization performance by leveraging pre-trained audio-to-landmarks predictors and audio encoders \cite{prajwal2020lip} respectively.
Although these initiatives successfully bolster the cross-audio generalization of NeRF methods, instances of cross-audio desynchronization persist.
\paragraph{Face Restoration.} Blind Face Restoration \cite{wang2023dr2,zhou2022towards,gu2022vqfr,zhu2022blind,he2022gcfsr,yang2024pgdiff,suin2024diffuse,wang2021towards,yang2021gan} endeavors to rectify diverse degradations in facial imagery captured in non-ideal conditions, including compression, blur, and noise. Given the unknown nature of the specific degradations, this task is inherently ill-posed. The prevalent approach to address this challenge is to integrate prior knowledge within an encoder-decoder framework:
GFPGAN~\cite{wang2021towards} utilizes a pre-trained face GAN's rich priors for blind face restoration. Yang et al.~\cite{yang2021gan} integrate a GAN into a U-shaped DNN for enhanced face restoration. Codeformer~\cite{wang2021towards} uses a learned codebook to predict codes, reducing restoration uncertainty. 
Through rigorous theoretical analysis and empirical experiments, we establish that the SOVQAE is an effective plug-and-play model for achieving identity-specific, temporally-consistent face restoration, thereby enabling the generation of high-quality talking faces.

\paragraph{Vector Quantization (VQ) Family.} 
Initially, VQ emerged as a technique within the Signal Processing community \cite{gray1984vector,makhoul1985vector,cosman1993using,farvardin1990study}. Mathematically, VQ maps signal spaces $\mathbb{R}^K$ onto a finite set of Voronoi regions $\{{V_n}\}_{n=1}^N$, each associated with an anchor vector $\{{\bf c}_n\}_{n=1}^N \subset \mathbb{R}^c$:
\begin{equation}
V_n=\{{\bf x} \in \mathbb{R}^c \mid \|{\bf x}-{\bf c}_n\|_2 \le \|{\bf x}-{\bf c}_m\|_2, \forall m \neq n\},
\label{Voronoi}
\end{equation}
where $\|\cdot\|_2$ denotes the Euclidean distance. (Fig.~\ref{fig_illus}(a) represents a Voronoi case in 2D space.) Recently, Oord \textit{et al.} \cite{van2017neural} integrated the VQ operation into generative models, sparking a trend of discrete latent representation in the field \cite{zhou2022towards,wang2023lipformer,esser2021taming,guo2023msmc,vali2023interpretable,sadok2024multimodal,jiang2023text2performer,rombach2022high,kaji2023vq}. 

Our study regards VQ-AE as a mechanism for learning and encoding high-frequency facial details, similar to VQ-based approaches. Unlike methods\cite{vali2023interpretable,sadok2024multimodal,jiang2023text2performer,rombach2022high,kaji2023vq} that require additional networks, VQ-AE inherently possesses noise robustness (Sec.\ref{sec:denosing_modeling}), enabling the recovery of detailed textures from low-quality images. This simplicity makes our method both effective and efficient.

\section{Noise Robustness of VQ-AE}
\label{sec:denosing_modeling}
In this section, we derive the noise robustness of VQ-AE via the Lipschitz Continuity theory, where the definition of lipschitz continuity is following:
\begin{definition}
\label{def_lip}
     A function $f: \mathbb{R}^n \to \mathbb{R}^m$ is called Lipschitz continuous if there exists a constant $L$ such that 
     \begin{equation}
         \forall x,y \in \mathbb{R}^n, ~\| f(x)-f(y)\|_2 \le L\|x-y\|_2.
     \end{equation} 
\end{definition}
The smallest value of \( L \) for which the preceding inequality holds is referred to as the Lipschitz constant of \( f \). This property guarantees that when a small perturbation \(\Delta x\) is applied to \( x \), the resulting impact on \( f(x) \) is linearly proportional to \(\|\Delta x\|\). This characteristic is more stringent than \textit{uniform continuity} because the Lipschitz constant is independent of specific points in the domain.

\begin{figure*}[t]
  \centering
  \includegraphics[width=1.0\textwidth]{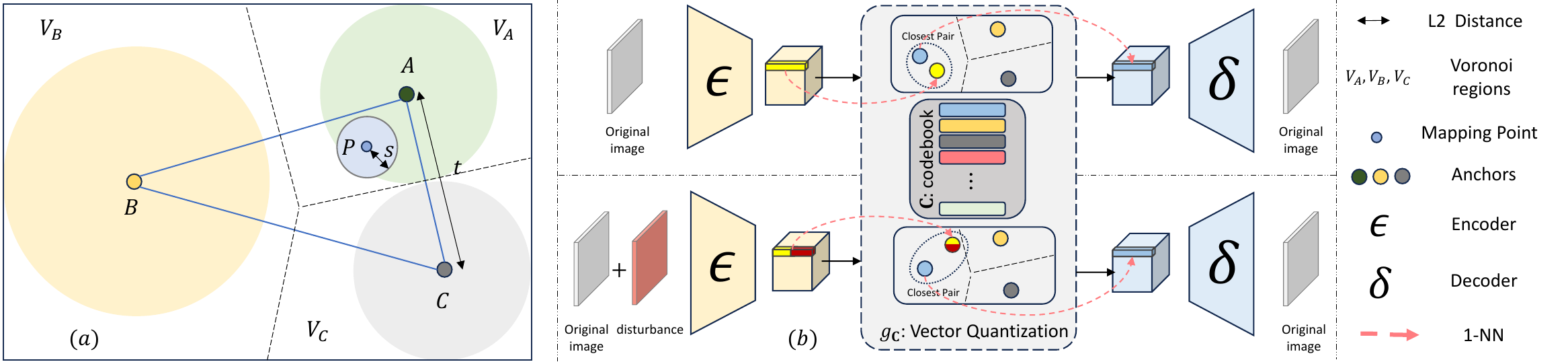}
  \vspace{-0.5cm}
  \caption{(a) Illustrates Voronoi regions in 2D and explains the denoising mechanism in VQ space. Given a mapping point $P$, there exists an open disk $\{ {\bf x}\in V_A\cup V_B\cup V_C : \|{\bf x}-P\|_2 <s \}$, where all points in the disk share the same anchor point $A$. 
  (b) Demonstrates the robustness of VQ-AE under slight input disturbances. When image perturbations occur, the VQ mechanism correctly matches the appropriate codebook vector in the latent space, provided the affected vector remains within the correct Voronoi region (e.g., the yellow-red vector in the figure).}
  \label{fig_illus}
\end{figure*}

As depicted in Fig.\ref{fig_illus}(b), a VQ-AE is encapsulated by the quadruple \( \{ \epsilon, \delta, {\bf C}, g_{{\bf C}} \} \). Here, \( \epsilon \) represents the CNN encoder that maps the image domain \( \mathbb{R}^{3 \times h \times w} \) to the latent space \( \mathbb{R}^{c \times h_o \times w_o} \). The term \( \delta \) signifies the CNN decoder, which reconstructs the image from the latent representation in \( \mathbb{R}^{c \times h_o \times w_o} \). The codebook \( {\bf C} = \{ {\bf c}_n \}_{n=1}^N \in \mathbb{R}^{c \times N} \) comprises \( N \) anchor vectors. Furthermore, \( g_{{\bf C}} \) refers to the channel-wise 1-Nearest Neighbor (1-NN) feature matching operation, responsible for the substitution of latent features with their closest anchor vectors within the codebook.
As demonstrated in recent studies \cite{balan2018lipschitz, zou2019lipschitz}, the following conclusion can be drawn:

\begin{theorem}
\label{theorem_1}
Given a trained VQ-AE $\{ \epsilon, \delta, {\bf C}, g_{{\bf C}}\}$, the $\epsilon : ~\mathbb{R}^{3 \times h \times w} \to \mathbb{R}^{c \times h_o \times w_o}$ is a map with lipschitz continuity such that
\begin{equation}
    \forall x,y \in \mathbb{R}^{3 \times h \times w}, ~\| \epsilon(x)-\epsilon(y)\|_F \le L_{\epsilon}\|x-y\|_F,
\end{equation}
where $\| \cdot \|_F$ is the Frobenius norm of matrix and $L_{\epsilon}$ is the Lipschitz constant of $\epsilon$.
\end{theorem}
We provide the proof in Appendix.\ref{append_3.1}. Let \({\bf V}_{low} \subset \mathbb{R}^{T \times 3 \times h \times w}\) denote the noised version of \({\bf V}_{high}\), where \({\bf V}_{high}\) training video of VQ-AE. 
Without loss of generality, for any \({\bf I}_{low} \in {\bf V}_{low}\) and \({\bf I}_{high} \in {\bf V}_{high}\), we assume that \({\bf I}_{high} + {\bf \mathcal{N}} = {\bf I}_{low}\) holds, where \({\bf \mathcal{N}}\) represents the numerical disturbance caused by any degradation. Consequently, by Theorem~\ref{theorem_1}, we derive the following results:
\begin{equation}
\epsilon({\bf I}_{low}) \in \{x \in \mathbb{R}^{c \times h_o \times w_o} : \|x-\epsilon({\bf I}_{high})\|_F \le L_{\epsilon}\|\mathcal{N}\|_F\}. 
\end{equation}
This implies that the $\epsilon(\mathbf{I}_{\text{low}})$ resides within the interior of a hypersphere, with the point $\epsilon(\mathbf{I}_{\text{high}})$ as its center and a radius of $L_{\epsilon}\|\mathcal{N}\|_F$, in the measure of the Frobenius norm. Furthermore, for each channel of the latent vector, denoted as $\epsilon(\mathbf{I}_{\text{low}})_{i,j}$ and $\epsilon(\mathbf{I}_{\text{high}})_{i,j} \in \mathbb{R}^c$, we have that:
\begin{equation}
\|\epsilon({\bf I}_{low})_{i,j}-\epsilon({\bf I}_{high})_{i,j}\|_2 < L_{\epsilon}\|\mathcal{N}\|_F.
\end{equation}
This further indicates an important and evident fact: 
\begin{theorem}
\label{theorem_2}
In a trained VQ-AE $\{ \epsilon, \delta, {\bf C}, g_{{\bf C}}\}$, we define $d_C={\rm min}\{\|{\bf c}_i-{\bf c}_j\|_2~:~{\bf c}_i,{\bf c}_j \in {\bf C} ~and~ i \neq j\}$ as the minimal distance between anchors of codebook ${\bf C}$, $\gamma$ is the maximal distance of training image latent to closest anchor in all channel, and $L_{\epsilon}$ is the Lipschitz constant of $\epsilon$ . when $\| \mathcal{N} \|_F < \frac{d_C-2\gamma}{2 L_{\epsilon}}$ holds, we have: 
\begin{equation}
    g_{\bf C}(\epsilon({\bf I}_{low}))=g_{\bf C}(\epsilon({\bf I}_{high})).
\end{equation}
\end{theorem}

This is what we refer to as the \textit{noise robustness/tolerance} mechanism of the VQ-AE, and we term the $\frac{d_C-2\gamma}{2 L_{\epsilon}}$ as Noise Robustness Upper Bound (NRoUB). Specifically, the anchors $\{ \mathbf{c}_n \}_{n=1}^N$ partition the latent domain into $N$ high-dimensional Voronoi regions as defined by Eq.~\eqref{Voronoi}. Let the nearest anchor to $\epsilon(\mathbf{I}_{\text{high}})_{i,j}$ be $\mathbf{c}_n$, and $V_n$ be its corresponding Voronoi region. Theorem~\ref{theorem_1} ensures that $\epsilon(\mathbf{I}_{\text{low}})_{i,j}$ remains within $V_n$, provided that $\| \mathcal{N} \|_F < \frac{d_C-2\gamma}{2 L_{\epsilon}}$ is satisfied, thereby enabling the $g_{\mathbf{C}}$ operation to correctly assign $\mathbf{c}_n$ to $\epsilon(\mathbf{I}_{\text{low}})_{i,j}$. This condition is valid for each channel of $\epsilon(\mathbf{I}_{\text{low}})$, and thus Theorem~\ref{theorem_2} applies, allowing the decoder $\delta$ to generate high quality faces without loss. The complete proof of Theorem~\ref{theorem_2} is detailed in Appendix~\ref{append_3.2}. 


As shown in Fig.~\ref{fig:nois_robustness}, we validate the noise robustness on a brief talking sequence with two common video degradation types: Gaussian blurring and Gaussian noise. For each frame, a $192 \times 192$ pixel area is randomly selected and subjected to these noise conditions. The increase in PSNR from 30.26 to 35.78 following Gaussian blurring confirms the temporal validity of Theorem~\ref{theorem_2}. Even under more semantically destructive Gaussian noise, our method adeptly recovers detail and maintains temporal consistency, with PSNR rising from 16.08 to 30.36. These observations underscore the robustness of VQ-AE to noise and affirm its effectiveness.

\paragraph{Application insights for High-Frequency Texture Repair in ID-specific THG.} 

On the one hand, in ID-specific THG, generating realistic audio-driven videos with high-frequency details is challenging to existing end-to-end methods. A common way to enhance face details is doing face restoration. However existing face restoration models fall short in temporal-consistency. 
Hence, a feasible post-processing method for ID-specific THG is to enhance texture with temporal-consistency.

On the other hand, our aforementioned theoretical analysis highlights two important and interesting conclusions: 

\begin{itemize}[leftmargin=20pt,parsep=2pt,itemsep=1pt,topsep=0pt]
    \item \textit{An ID-specific denoiser can be obtained through training the ID-aware neural discrete representation (VQ-AE);}
    \item \textit{The denoising performance can be temporal-consistent, when the noise disturbance is below the NRoUB.}
\end{itemize}

Therefore, if we can increase the NRoUB of VQ-AE, we can obtain the ideal post-processing model for ID-specific THG. To this end, we introduce a space regularization loss to enhance the NRoUB of VQ-AE (termed as  SOVQAE) in next section, and illustrate its combination with pretrained Wav2Lip showing its application potential.


\section{Space-Optimized VQ-AE and Application}
\label{sec:method}
In this section, we firstly introduce a feasible space regularization loss to enhance the NRoUB, building our SOVQAE, and illustrate how the pretrained Wav2Lip combines with our SOVQAE to build state-of-the-art ID-specific THG pipeline. 

\begin{wrapfigure}{r}{0.6\textwidth}
  \centering
  \includegraphics[width=0.6\textwidth]{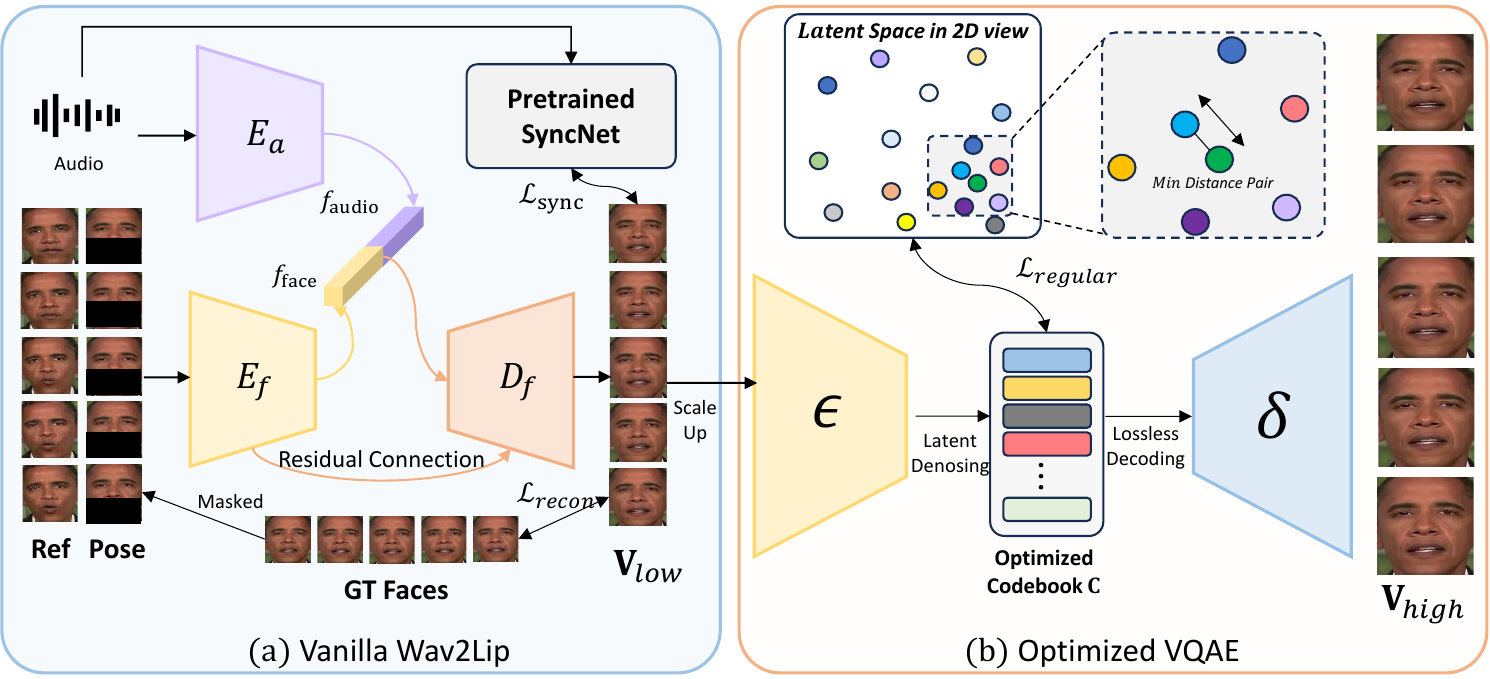}
  \vspace{-0.2cm}
  
  \caption{Illustration of cascade pipeline: (a) displays the overview of  Wav2Lip; (b) displays the latent denosing process of SOVQAE and the mechanism of regularization loss. Please zoom in for better visualization.}
  \label{fig_overview}
\end{wrapfigure}

\subsection{Space Optimized VQ-AE}

According to Theorem~\ref{theorem_2}, perfect denoising is achievable when $\|\mathcal{N}\|_F < \frac{d_C-2\gamma}{2 L_{\epsilon}}$. Hence our goal is increasing the NRoUB thus to enhance the noise robustness of VQ-AE. {Note} that $\gamma$ is much close to zero, because $\gamma$ is optimized via VQ loss (see Eq.~\ref{loss_vq}) in training. Hence, it seems optimal to maximize $d_C$ while minimizing $L_{\epsilon}$. 
However, this is theoretically infeasible, because $L_{\epsilon}$ is positively correlated with the operation norm of convolutional kernels (according to our proof of Theorem \ref{theorem_1}), implying that minimizing $L_{\epsilon}$ equates to L2/L1 parameter regularization~\cite{goodfellow2016deep}. Such regularization would globally compress the latent space distribution, consequently compressing the distribution of anchors in the codebook $\mathbf{C}$ and reducing $d_C$, thus to degrade NRoUB, as demonstrated in Tab.~\ref{tab:ablation_reg}.

Therefore, the collect optimization is that increasing $d_C$ while maintaining $L_{\epsilon}$. 
In experiments, we find a practical solution is to \textit{locally} extend the $d_C$ with a given lower bound. Thus, the codebook regularization loss we propose is:
\begin{equation}
\mathcal{L}_{\text{regular}} = \|d_{C} - \theta\|_2,
\label{loss_regular}
\end{equation}
where $\theta$ represents the given lower bound. This regularization strategy effectively improves local noise robustness while maintaining global robustness, which will be validated via ablation studies.
In addition to this regularization loss, the complete loss function of SOVQAE comprises L2 reconstruction loss, VQ loss\cite{esser2021taming}, perceptual loss, and GAN loss with a patch-based discriminator\cite{isola2017image}. The VQ loss is following:
\begin{equation}
\begin{aligned}
    \mathcal{L}_{\text{VQ}}(\epsilon,\delta,\mathbf{C})&=\|x-\hat{x}\|^2+\|\text{sg}[\epsilon(x)]-g_{\textbf{C}}(\epsilon(x))\|_2^2 \\
    &+\|\text{sg}[g_{\textbf{C}}(\epsilon(x))]-\epsilon(x))\|_2^2,
    \label{loss_vq}
\end{aligned} 
\end{equation}
where $x$ is the input image, $\hat{x}=\delta(g_{\textbf{C}}(\epsilon(x)))$ is the reconstructed image, $\text{sg}(\cdot)$ denotes the stop-gradient operation following\cite{esser2021taming}. The later two components of Eq.~\ref{loss_vq} are utilized to optimized the $\gamma$.




\subsection{Cascade ID-specific THG pipeline}

\noindent{\textbf{The Choice of Base Model.}}
In ID-specific THG, an important demanding is the generality with Out-of-Domain (OOD) audios. While end-to-end methods~\cite{ye2023geneface++, peng2023synctalk} make significant progress on it, they still have gap with great OOD audio generality, as their limited training data.  Furthermore, while Wav2Lip~\cite{prajwal2020lip} has been proposed for years, it still outperforms with both recent one-shot~\cite{tan2024edtalk, kim2025moditalker} and ID-specific~\cite{xie2025pointtalk, yu2024gaussiantalker} methods and run in real time. It is an ideal lip-sync prior base model for a post-processing framework. As shown in Tab.~\ref{tab:origin_voice}, this combination makes outstanding lip-sync performance on OOD audios, validating our choice of base model.

\noindent{\textbf{Cascade Pipeline.}}
As shown in Fig.~\ref{fig_overview}, Our SOVQAE is Cascaded with pretrained Wav2Lip, and we kindly refer readers to obtain more details from \cite{prajwal2020lip} due to the limited space.
To adapt with outputs of Wav2Lip, we employ face detection using S3FD \cite{zhang2017s3fd} to extract facial regions from a talking head video of a specific identity. These extracted regions are subsequently resized to $256 \times 256$ for training SOVQAE. This strategy allows the codebook to concentrate on capturing the high-frequency facial components. In inference, for given Low resolution faces ($96\times 96$) from Wav2Lip, we upscale them to $256 \times 256$ before subjecting them to SQVQAE. 

\noindent{\textbf{Efficiency and Plug-and-Play.}} With the \textit{half precision} training and inference, our cascade pipeline achieves 30 FPS inference speed on consumer GPU, which demonstrates its feasibility in industry. Total training takes 5 GPU hours on NVIDIA 4090, which is also comparable with end-to-end counterparts. As shown in Fig.~\ref{fig:app_add_exp}, trained SOVQAE can also adapt for different base models, since it works as a denoiser.

\begin{figure}[h]
  \centering
  \includegraphics[width=0.95\textwidth]{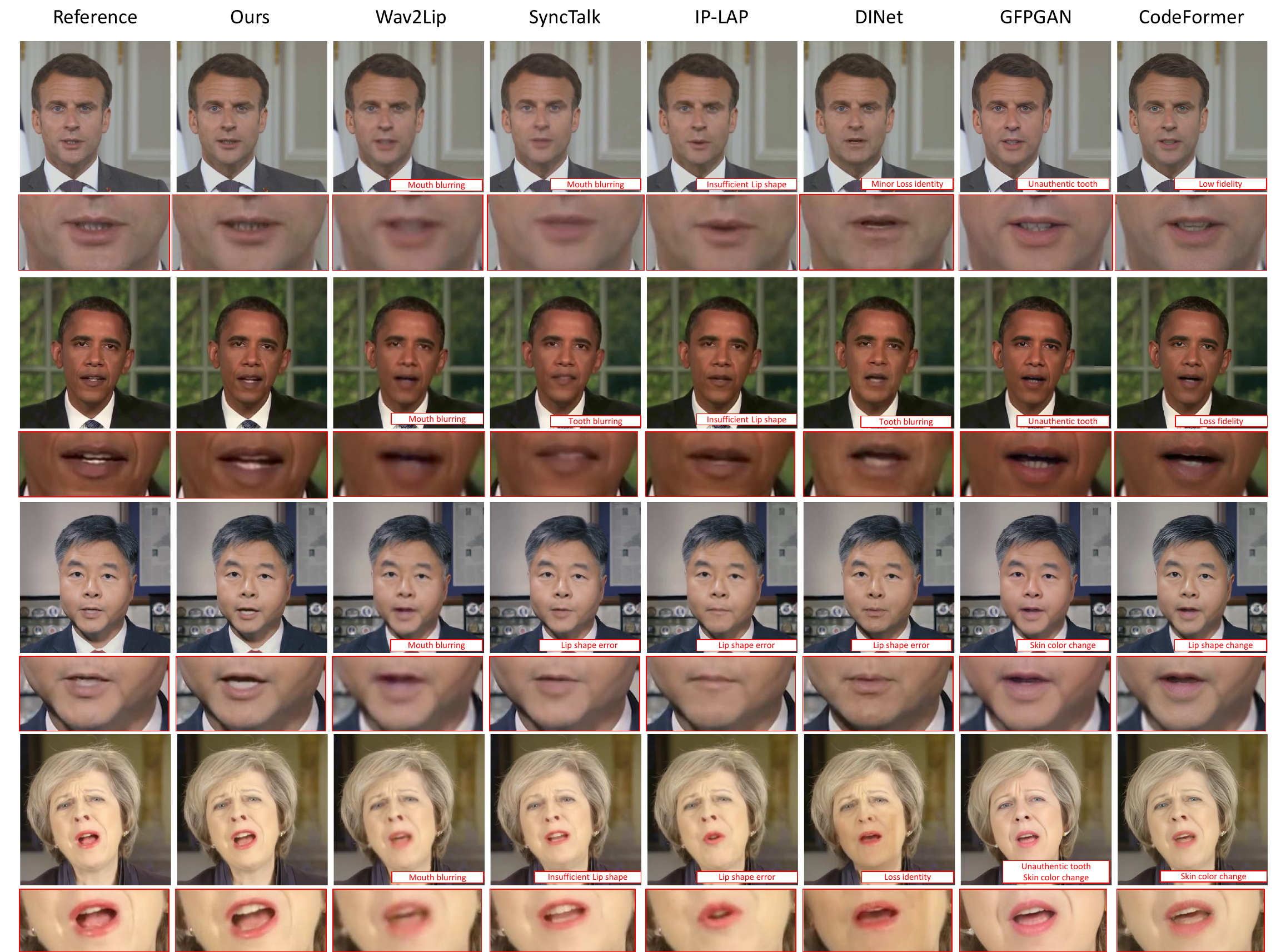}
  \vspace{-0.2cm}
  \caption{The comparative visualization presents a curated selection of discrete keyframes and corresponding sub-frames of lower-half faces. Please zoom in for better visualization.
  }
  \label{fig:qualitative}
  \vspace{-0.5cm}
\end{figure}

\section{Experiments}
\textbf{Dataset.} For comparative analysis, we employ the same set of meticulously edited video sequences utilized in previous works \cite{guo2021ad,li2023efficient,ye2023geneface}, encompassing both English and French dialogues. The average duration of these videos is approximately 8,843 frames. We refer to this collection as the \textit{Usual Dataset} to differentiate it from our proprietary \textit{HF videos}. These videos are characterized by their rich textures and high resolution ($1024\times 1024$). The average length of these videos is around 11,000 frames. More details of them can be found in Appendix.
All videos are centered on individual characters and captured at a frame rate of 25 FPS. Consistent with prior research, we adopt a train-to-test split ratio of $10:1$. We empirically set $\theta = 1$. Moreover, to assess the generalizability of our methods, we conduct additional tests using a set of out-of-domain (OOD) audio. This OOD audio comprises 10 diverse voice samples, spanning various languages and genders. \textit{We place more implementation details in Appendix, due to the limited space.}

\textbf{Comparison Baselines.}
We conduct a comparative analysis of our method with a total of three one-shot approaches, which include Wav2Lip \cite{prajwal2020lip}, DINet \cite{zhang2023dinet}, IP-LAP \cite{zhong2023identity}, and three ID-specific THG methods, namely SyncTalk \cite{peng2023synctalk}, GeneFace \cite{ye2023geneface} and GeneFace++ \cite{ye2023geneface++}. Furthermore, to underscore the significance of our identity-specific face restoration concept, we establish two post-processing baselines using Codeformer \cite{zhou2022towards} and GFPGAN \cite{wang2021towards}.

\textbf{Evaluation Metrics.}
To assess the quality of the generated images, we employ the Peak Signal-to-Noise Ratio (PSNR) and the Frechet Inception Distance (FID) \cite{heusel2017gans} as metrics. For evaluating lip synchronization, we utilize the Lip Sync Error Confidence (LSE-C) and Lip Sync Error Distance (LSE-D) scores.

\subsection{Qualitative Comparison}
To provide a more intuitive assessment of image quality, we present a comparative analysis of our method alongside alternative approaches in Fig.~\ref{fig:qualitative}. Specifically, we contrast our method with both end-to-end methodologies, including Wav2Lip~\cite{prajwal2020lip}, SyncTalk~\cite{peng2023synctalk}, IP-LAP~\cite{zhong2023identity}, and DINet~\cite{zhang2023dinet}, as well as two post-processing techniques, GFPGAN~\cite{wang2021towards} and CodeFormer~\cite{zhou2022towards}.

\begin{wrapfigure}{r}{0.5\textwidth}
  \centering
  \includegraphics[width=0.5\textwidth]{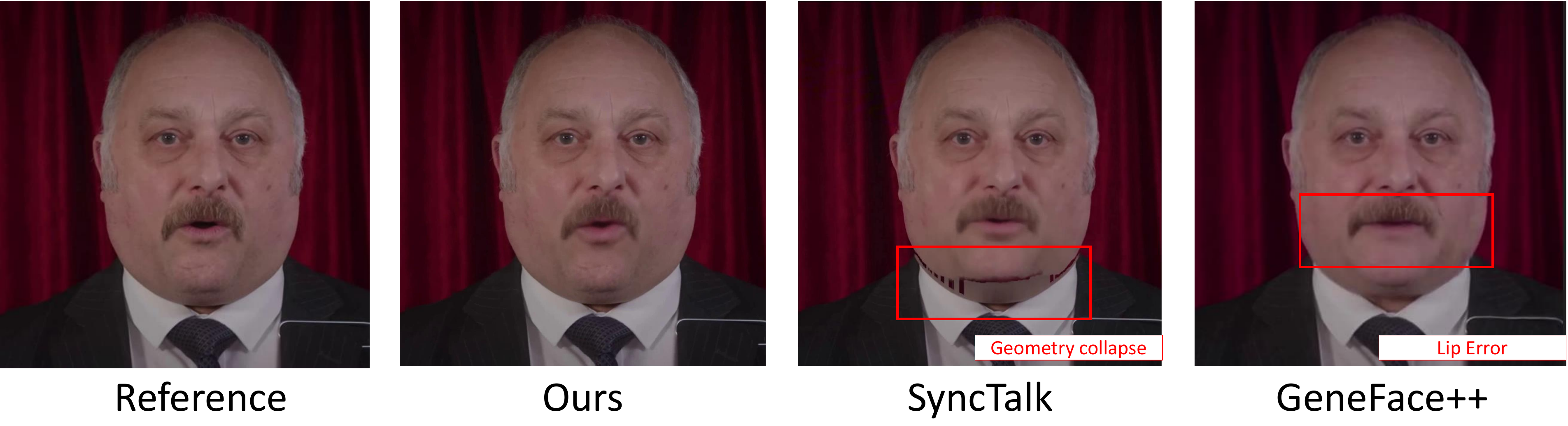}
  \vspace{-0.5cm}
  \caption{ Comparison on HF videos. Except for the collapse and lip error (see red boxes), NeRF-based methods also has limitation in stability of face scale. See supplemental videos for more details.
  }
  \label{fig:1080}
\end{wrapfigure}
In summary, the qualitative comparison underscores that our method surpasses all compared methods in terms of facial texture and lip-synchronization. 
In Fig.~\ref{fig:qualitative}, our method outperforms SyncTalk, a state-of-the-art NeRF-based approach, by better capturing high-frequency facial details. SyncTalk's struggle to maintain these details, particularly in dynamic scenes, results in blurred mouth and teeth regions. This is attributed to NeRF's limited expressivity for high-frequency details. In contrast, our method, utilizing SOVQAE, effectively restores individual-specific textures, enhancing fidelity.
The 7th and 8th columns of the figure, where our method's identity-specific restoration approach is superior to the two post-processing baselines, which alter the subject's identity due to information leakage from a large OOD faces dataset~\cite{karras2019style}. This underscores the necessity of our identity-specific restoration approach.
Besides, in more challenging HF videos, our method has more stable performance. In contrast, NeRF-based methods occur geometry collapse (see the 3-th cloumn in Fig.\ref{fig:1080}) and lip error, when scale their resolution to higher (from 512 to 1024).


\subsection{Quantitative Comparison}
\begin{wrapfigure}{r}{0.4\textwidth}
\vspace{-0.5cm}
  \centering
  \includegraphics[width=0.4\textwidth]{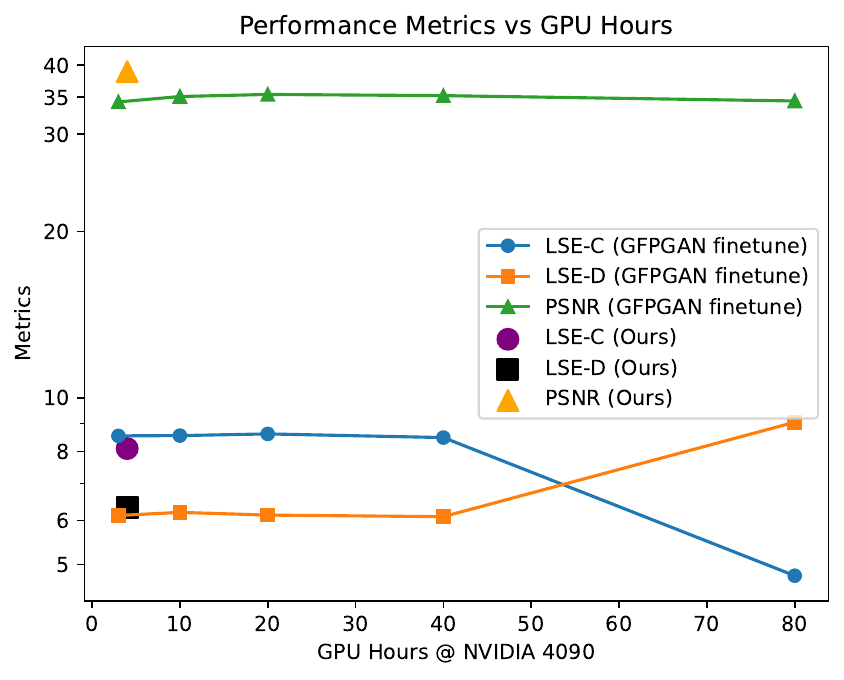}
\vspace{-0.5cm}
  
  \caption{ Comparison with fine-tuned GFPGAN model.Note that the sudden collapse of lip synchronization of GFPGAN is caused by the overfitting. We display videos in our \textit{supplemental materials}.
  }
  \label{fig:finetune}
\end{wrapfigure}
For a more exhaustive evaluation, our quantitative assessment is divided into two segments: 
1) \textit{Video quality and lip-synchronization comparison}, which evaluates image quality frame-by-frame and lip-synchronization performance under original audio-driven conditions; 
2) \textit{Out-of-Distribution (OOD) audio-driven comparison}, where we contrast lip-synchronization with SOTA methods to underscore the superior cross-audio generality of our approach, a critical application scenario. 

\textbf{Comparison with End-to-End Methods.} 
As detailed in Tab.~\ref{tab:origin_voice}, on the Usual Dataset, our method secures the highest scores for both PSNR (38.9213 ) and FID (3.5502), signifying that we achieve state-of-the-art image quality for generated videos. 
This advancement is credited to the SOVQAE, which retains high-frequency texture details in the codebook and consistently recovers facial textures from low-quality faces. 
In terms of lip-synchronization, we outperform all methods (excluding our base model), a testament to Wav2Lip's robust audio-lip alignment capability. 
Although our method moderately reduces the lip-synchronization performance of the base model, the substantial enhancement in visual quality justifies this trade-off as valuable and insightful. 
Similar outcomes are observed in the HF videos.

\textbf{Comparison with Post-Process Methods.}
As illustrated in Tab.~\ref{tab:origin_voice}, our method surpasses two post-process baselines in video quality. 
This advantage stems from the fact that Codeformer and GFPGAN are tailored for blind face restoration, and their face prior learning from out-of-distribution multi-face datasets results in information leakage, leading to identity errors. 
Moreover, their image-by-image training strategy is not optimal for video scenarios, causing frame-wise flickering. 
Although they maintain the lip-synchronization performance of the base model, their subpar visual fidelity limits their applicability to ID-specific THG. 
Furthermore, for fair comparison, we also finetune GFPGAN model in ID-specific videos. As shown in Fig.~\ref{fig:finetune}, while fine-tuning improves its performance in visual quality, its synchronization decreases. On the other hand, our method has least GPU need (5 GPU@4090 hours), best visual quality and comparable synchronization.
These findings strongly affirm the efficacy and efficiency of our cascade pipeline for ID-specific THG.

\textbf{Comparison on OOD Audio Driving.} 
To evaluate our method's voice generalization, we compared it with recent SOTA models on the OOD audio-driving task. As shown in Tab.~\ref{tab:origin_voice}, Our method achieves second best in both LSE-C and LSE-D metrics, scoring \textbf{6.5521} in LSE-C and \textbf{7.6231} in LSE-D, demonstrating strong generalization across diverse voices. It significantly outperforms other methods such as SyncTalk, GeneFace++, IP-LAP, and DINet, and is slightly behind Wav2Lip.

\begin{table}[t]
\centering
\caption{The quantitative experiments of different methods on Usual Dataset and HF videos, and the OOD audio-driven experiment. * indicates post-process baselines combined with Wav2Lip.}
\resizebox{1.0\textwidth}{!}{
\begin{tabular}{ccccccccccc}
\toprule
{\textbf{Method}} & \multicolumn{4}{c}{Usual Dataset} & \multicolumn{4}{c}{HF videos} & \multicolumn{2}{c}{OOD-audio}\\
\cmidrule(r){2-5} \cmidrule(r){6-9}  \cmidrule(r){10-11}
~& PSNR$\uparrow$ & FID$\downarrow$ & LSE-C$\uparrow$ & LSE-D$\downarrow$ & PSNR$\uparrow$ & FID$\downarrow$ & LSE-C$\uparrow$ & LSE-D$\downarrow$ & LSE-C$\uparrow$ & LSE-D$\downarrow$\\
\midrule
Wav2Lip\cite{prajwal2020lip} & 34.0979 & 10.8503 & \underline{9.077} & 5.8769 & 34.6593 & 17.5281 & \textbf{8.8496} & {6.8712} &\textbf{7.1258} & \textbf{7.3521}\\
DINet\cite{zhang2023dinet} & 33.9108 & 9.8385 & 7.1951 & 7.4343 & 34.4849 & 10.5749 & 8.3541 & 7.2340  &5.2310 &8.5922\\
IP-LAP\cite{zhong2023identity} & 34.7263 & 10.1703 & 5.1736 & 8.8160 & 35.1595 & 9.9438 & 8.2604 & 7.1838 &4.3759  &9.0596 \\
GeneFace\cite{ye2023geneface} & 24.8165 & 21.7084 & 5.195 & -&  -&  -& - &  - & - & -\\
GeneFace++\cite{ye2023geneface++} & 31.1164 & 20.5506 & 6.8916  &7.5014 &33.9692  &22.3928   &4.6112  &11.0401  & 5.3445& 8.3927\\
SyncTalk\cite{peng2023synctalk} & 36.0574 & 6.4855 & 7.054 & 7.282 & 39.4050 & 6.5370 & 7.816 & 7.715  &5.1063 & 8.0474 \\
\midrule
CodeFormer*\cite{zhou2022towards}&33.2441&26.3567&\textbf{9.1739}&\textbf{5.7720}&33.7958&16.7215&7.9409& 7.2582 & -&-\\
GFPGAN*\cite{wang2021towards}&33.9803&16.7267&9.0417&\underline{5.8144}&33.7780&17.2715&\underline{8.7940}&\textbf{6.5193} & -&-\\
\midrule
Ours & \textbf{38.9213} & \textbf{3.5502} & {8.1136} & {6.3402} & \textbf{39.6254} & \textbf{4.3525} & {8.2392} & \underline{6.8238}  & \underline{6.5521}  & \underline{7.6231}\\
\bottomrule
\multicolumn{8}{l}{\textbf{Best}; \underline{Second best}.}
\end{tabular}}
\vspace{-1.0em}
\label{tab:origin_voice}
\end{table}

\begin{wraptable}{r}{0.5\textwidth}
\vspace{-0.5cm}
    \caption{Plug-and-play experiments in Usual Dataset.}
    \centering\resizebox{0.5\textwidth}{!}{
    \begin{tabular}{c|c|c|c|c}
    \toprule
    Method& PSNR$\uparrow$ & FID$\downarrow$& LSE-C$\uparrow$ & LSE-D$\downarrow$ \\
    \hline
    GeneFace++&31.11&20.55&6.89&7.50\\
    GeneFace++ + SOVQAE &33.27&18.21&6.42&8.22\\
    \hline
    SyncTalk (CVPR 2024)    &36.06&6.49&7.05&7.68\\
    SyncTalk + SOVQAE    &37.16&5.12&6.58&7.88\\
    \hline
    Ours &\textbf{38.92} & \textbf{3.55} & \textbf{8.11} & \textbf{6.34}\\
    \bottomrule
    \end{tabular}}
    \label{tab:plug-and-play}
\end{wraptable}
\textbf{Plug-and-Play Performance.}
To demonstrate the plug-and-play performance of SOVQAE, we integrated SOVQAE into different ID-specifc base models. As shown in Tab.\ref{tab:plug-and-play}, SOVQAE improves visual quality metrics across different ID-specific base models, confirming its versatility and effectiveness. Notably, we achieve the best overall performance by leveraging the exceptional audio-lip synchronization capabilities of Wav2Lip. This superior synchronization positively impacts PSNR and FID values, thereby enhancing overall video quality and ensuring ours top performance.


    
    

These observations suggest that our method effectively inherits the SOTA audio generalization capability of Wav2Lip and validate the feasibility of our cascade pipeline. 

\subsection{Ablation Study} 

In this section, we conduct an ablation study to further substantiate the indispensability of each component within our pipeline.


\textbf{VQ Regularization Loss.}
As previously discussed, in practical applications, training without regularization on the VQ codebook can result in a decrease in noise robustness due to the stochastic optimization of the codebook. 
While Fig.~\ref{fig:nois_robustness} demonstrates impressive robustness under Gaussian blur and Gaussian noise, vanilla VQ-AE falls short to handle the unknown noise disturbance between the output of Wav2Lip and ground truth, resulting artifacts.
To illustrate the efficacy of Eq.~\ref{loss_regular} in stabilizing this robustness, we conducted a comparison between the VQ-AE with and without codebook optimization. As shown in Fig.~\ref{fig:ablation_reg}, the absence of our regularization loss introduces numerous artifacts in the output video, attributable to errors in latent matching.

\begin{wrapfigure}{r}{0.5\textwidth}
  \centering  \includegraphics[width=0.5\textwidth]{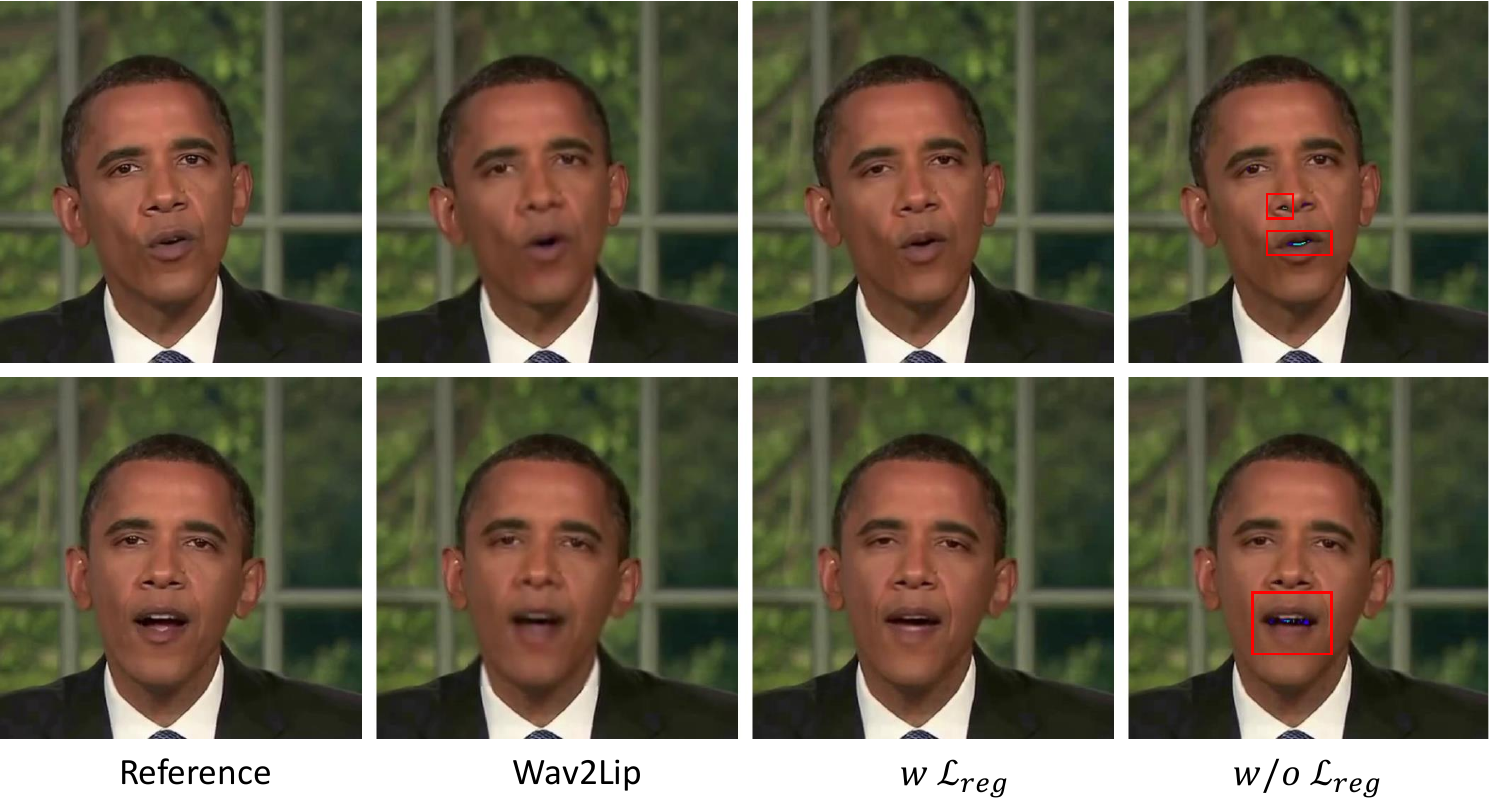}
  \vspace{-0.5cm}
  \caption{Ablation study of the regularization loss. Without this loss will lead to random noises ( see red boxes ).}
  \label{fig:ablation_reg}
  \vspace{-0.5cm}
\end{wrapfigure}
To further demonstrate the necessity and efficacy of our VQ regularization loss, we conduct detailed ablation studies on the lower bound $\theta$ and optimization settings. As shown in Tab.~\ref{tab:ablation_reg}, optimizing the minimal distance pair in the codebook yields the best results: 
(1) Optimizing pairwise distances within the codebook leads to codebook collapse, as indicated by the first two rows. 
(2) In contrast, local optimization with $\theta=1$ is effective. 
(3) Increasing the distance ($\theta=2$) in local optimization degrades performance due to over-constraint. 
This contrasts with general principles for training VQ-VAEs, where maintaining a sufficiently large distance in the codebook is crucial for diversity in image representation. However, our goal differs: while VQ-VAE aims for diverse codebooks to handle varied images, SOVQAE focuses on enhancing noise robustness for identity-specific faces.
We do not explore $\theta<1$ since $d_c<1$ is typically the case. 
\begin{wraptable}{r}{0.5\textwidth}
    \caption{Ablation experiments in Usual Dataset. `Reg. Obj.' indicates the regularized object in our regularization loss. `Minimal Dist.' is the $d_c$. `Average Dist.' indicates the average distances of each pair of codebook vectors.}\centering\resizebox{0.5\textwidth}{!}{
    \begin{tabular}{c|c|c|c|c|c}
    \toprule
     PSNR$\uparrow$ & FID$\downarrow$& LSE-C$\uparrow$ & LSE-D$\downarrow$ & Reg. Obj.& $\theta$\\
    \hline
     31.21&29.34&4.35&9.58&Average Dist.&2\\
     34.18&18.34&5.67&8.02&Average Dist.&1\\
    34.56&19.27&5.35&8.26&Minimal Dist.&2\\
    \textbf{38.97} & \textbf{3.45} & \textbf{8.13} & \textbf{6.44} &Minimal Dist.&1\\
     \bottomrule
    \end{tabular}}

    \label{tab:ablation_reg}
\end{wraptable}

\textbf{Length of training videos.}
An essential inquiry pertains to the least length of video required to train a sufficiently expressive SOVQAE to support our pipeline in achieving realistic and temporally consistent denoising performance. To address this question, we experimentally varied the length of the training video for the SOVQAE and assessed the corresponding impact on performance. The findings, as detailed in Tab.~\ref{tab:ablation_length}, yield the following insights: 
Using extended training videos improves the audio-driven capabilities of SOVQAE. A 5-minute training video is adequate for generating realistic talking head videos with our method, indicating similar data requirements to NeRF-based methods. To make our manuscript more sufficient, we introduce our limitations in Appendix~\ref{discussion}.

\begin{table}[h]
    \centering
    \caption{Ablation study on length of training video. \textbf{Note} that we only compute the PSNR for the \textbf{Mouth} region in this experiment. We conduct this experiment on \textit{Macron} for more comprehensive study, since it is the longest video in Usual dataset. }
    \resizebox{0.75\textwidth}{!}{\begin{tabular}{c|c|c|c|c|c|c|c}
        \toprule
        Length of video & 1 & 2 & 3 & 4 & 5 & 6 & Full (\textit{8'40''}) \\
        \hline
        LSE-C$\uparrow$ & {6.4420} & {6.9288} & {7.0556} & {6.9871} & {7.1377} & \textbf{{7.3170}} & \underline{{7.2912}} \\
        LSE-D$\downarrow$ & {7.2267} & {6.6919} & \underline{6.4644} & {6.5941} & {{6.5378}}& {6.5669} & \textbf{{6.3571}} \\
        Mouth PSNR$\uparrow$ & {29.5289} & {30.4725} & {31.0795} & {31.7604} & \underline{{32.3984}} & 32.2013& \textbf{{32.48622}} \\
        \bottomrule
    \end{tabular}}
    \label{tab:ablation_length}
\end{table}

\vspace{-0.5cm}
\section{Conclusion}
This work first establishes a theoretical foundation for the noise robustness of VQ-AE by proving an upper bound based on Lipschitz continuity theory, which enables subsequent method design. 
Building on this foundation, we propose SOVQAE, enhancing temporal consistency in latent space via a space regularization loss.
Furthermore, we introduce an efficient cascade pipeline combining a pre-trained Wav2Lip model with ID-specific SOVQAE, achieving state-of-the-art video quality and out-of-domain lip synchronization accuracy.
The whole pipeline requires only 5 GPU hours (NVIDIA 4090) for training and supports 30 FPS real-time inference, making it suitable for industrial deployment. 
In other words, this work highlights the potential of post-processing paradigms in the domain of talking head generation. We hope that our contributions will provide valuable insights to the research community.

{
    \small
    \bibliographystyle{unsrt}
    \bibliography{main}
}
\newpage
\appendix
\section{Implementation Details.}
In our experiments, we configured the codebook with $c = 256$ channels and a size of $N = 1024$. For the Usual Dataset, we trained the optimized VQ-AE on face crops resized to a resolution of $256 \times 256$ pixels, completing 50 epochs using an NVIDIA 4090 GPU, which amounted to 5 GPU hours. For the HF videos, the face crops were resized to a higher resolution of $512 \times 512$ pixels, and the model was trained for 50 epochs on the same hardware,  consuming 20 GPU hours. We utilized a downsampling factor of $h_o = h / 16$ and $w_o = w / 16$. Additional details regarding the network architecture are presented in our Appendix~\ref{append_more_detail}. The foundational Wav2Lip model employed in our experiments is a pretrained version sourced from the official repository\footnote{\url{https://github.com/Rudrabha/Wav2Lip}}. For the regularization loss defined in Eq.~\ref{loss_regular}, we have empirically set $\theta = 1$. The high frequency videos consists of 4 4K videos from YouTube, reprocessed to $1024\times1024$ at 25 fps with a 10M bitrate (compared to 2M in the Usual dataset). They have \textbf{higher resolution} and \textbf{richer high-frequency textures}, as shown in Fig.\ref{fig:hfvideos}. These videos enable evaluation of identity-specific methods under more detailed and realistic conditions, aligning with industrial needs.
\begin{figure}[h]
    \centering
    \includegraphics[width=1.0\linewidth]{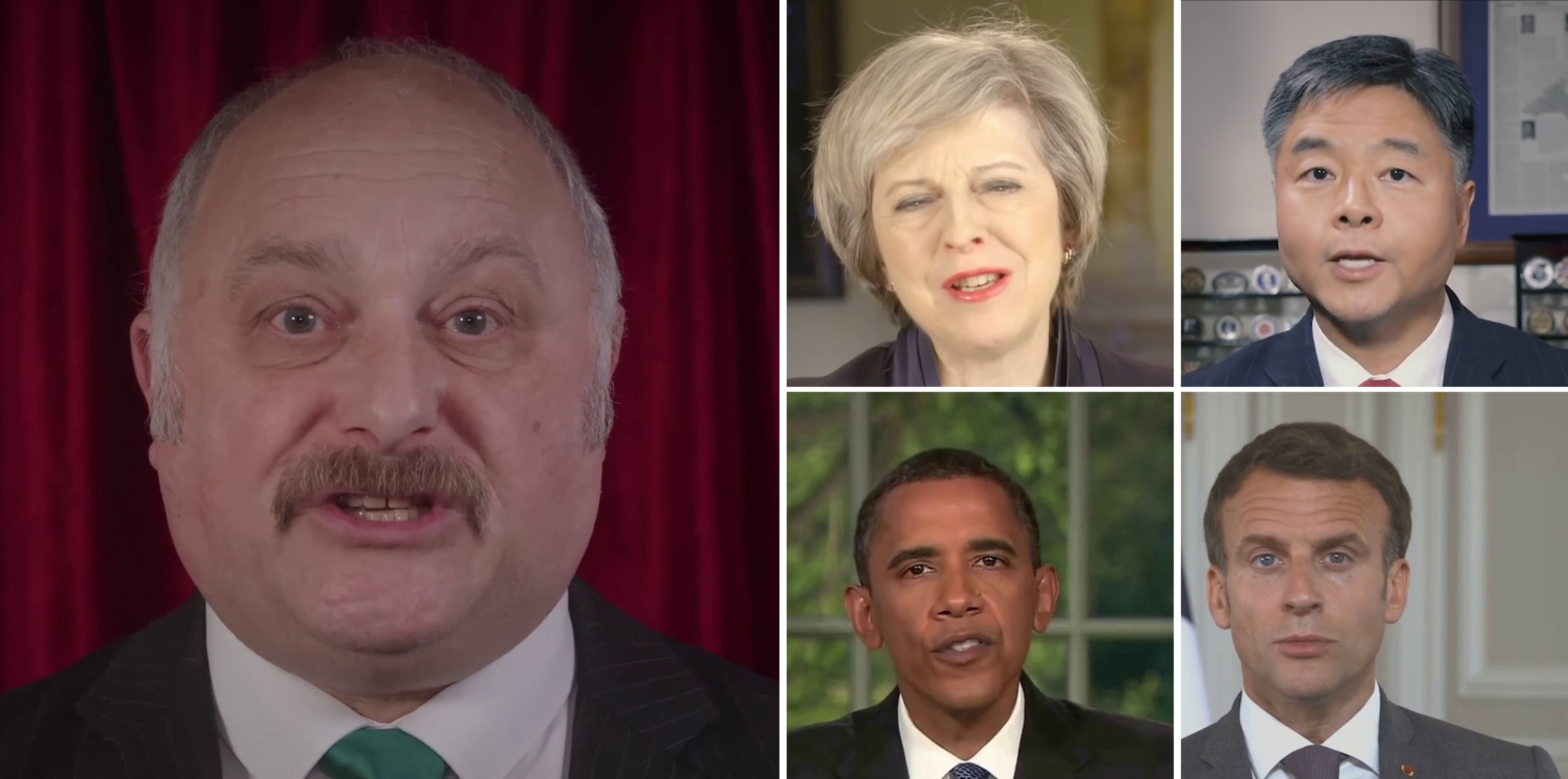}
    \caption{Example display of the processed HF videos. In the left side, we display one keyframe of one of HF videos. In the right side, we display four keyframes of videos in Usual Dataset. It is clear that HF videos have more fine-grained details, sharper textures and larger face region, which are more challenging. }
    \label{fig:hfvideos}
\end{figure}

\begin{figure}[h]
\vspace{-0.2cm}
  \centering
  \includegraphics[width=0.6\linewidth]{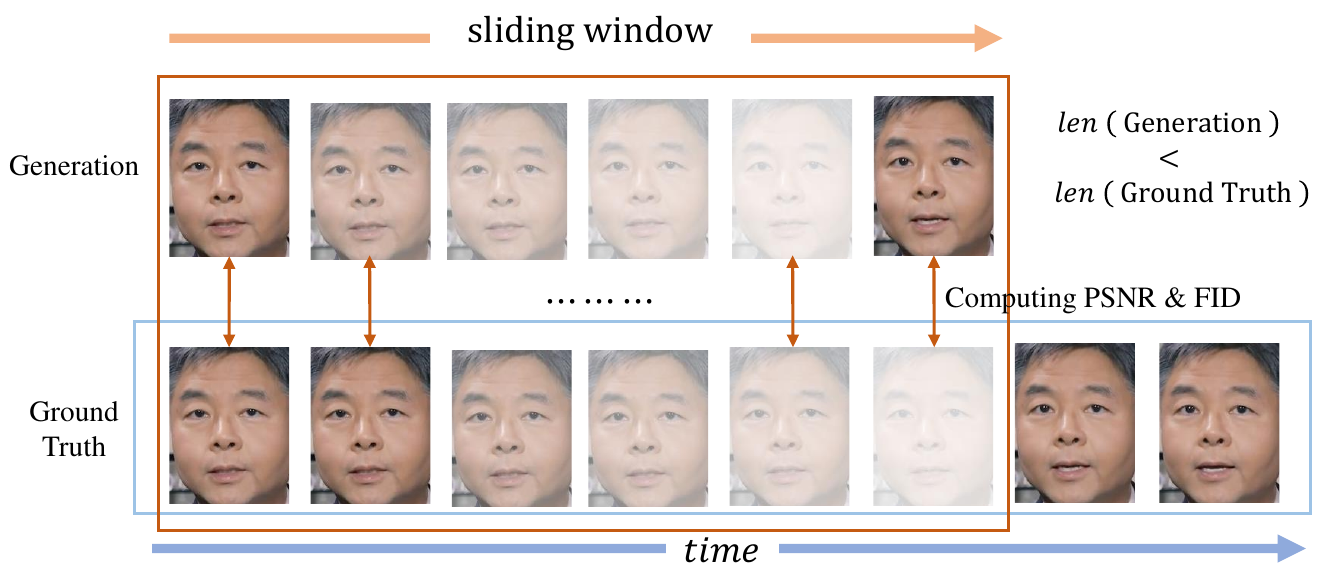}
  \caption{Illustration of our Sliding Evaluation method.}
  \label{fig_sliding}
  \vspace{-0.2cm}
\end{figure}
\textbf{Sliding Evaluation.}
In our experiments, we observe that the generated video lengths from various methods \cite{prajwal2020lip,zhang2023dinet,ye2023geneface,peng2023synctalk,li2023efficient} are not only disparate but also consistently shorter than those of the ground truth videos. Directly assessing image quality in conjunction with video duration could potentially diminish the perceived performance of these methods. To address this issue, we introduce a Sliding Evaluation technique. As illustrated in Fig.~\ref{fig_sliding}, this approach treats the generated video sequence as a dynamic sliding window, calculating the average image quality across all frames captured within the window. The window progressively moves along the timeline, and the maximum value obtained from these iterations is adopted as the final evaluation outcome. We apply this Sliding Evaluation method to all PSNR and FID results to evaluate different methods, thus ensuring a more holistic and equitable assessment.

\section{Discussion}
\label{discussion}
\subsection{Limitations}

In comparison with existing methods, our approach has some limitations in terms of computational requirements for both training and inference phases. While NeRF-based methods such as SyncTalk \cite{peng2023synctalk} and GeneFace++ can achieve faster inference (>40 FPS), our pipeline has only 30 FPS inference speed on an Nvidia 4090 GPU. Besides, some 3DGS-based methods could even achieve ultra inference speed (>100 FPS). Additionally, our training demand is more substantial; for instance, to train our model on 5 minutes of talking head video data, we require approximately 5 GPU hours using an NVIDIA 4090 GPU. These computational demands highlight areas for potential future optimization.
On the other hand, our pipeline cannot achieve novel view synthesis, as it is a video post-processing pipeline. 

Nevertheless, we believe our approach remains valuable for repairing high-frequency details in identity-specific talking-head generation (ID-specific THG).

\subsection{Boarder Impacts}

\subsubsection{Postive impacts}
In contrast to traditional end-to-end frameworks, our method generates realistic talking head videos through a temporally-consistent post-processing approach. From an alternative viewpoint, we also highlight the potential of harnessing the robust lip synchronization capabilities of a pretrained Wav2Lip model. This suggests that integrating a pretrained Wav2Lip model into a NeRF-based talking head pipeline could significantly enhance per-frame stability and synchronization performance.
Moreover, a primary contribution of our work is the theoretical proof and practical demonstration of the noise robustness inherent in the Vector Quantization  mechanism. Consequently, it is a logical next step to explore the application of this concept within the context of adversarial attacks \cite{madry2017towards, xu2020adversarial, zeng2019adversarial} to bolster the security and reliability of neural networks.

\subsubsection{Negative impacts.}
Our work proposes a method for generating high-fidelity, identity-specific talking-head videos. While such technology has promising applications in entertainment, virtual assistants, and telecommunication, it also raises critical ethical concerns. For example, our method could be misused to create highly realistic fake videos for malicious purposes, such as identity theft, political manipulation, or financial fraud. The realism of ID-specific generation may reduce public trust in video content. Hence, if our model is released publicly, we will require users to agree to a usage policy prohibiting misuse (e.g., deepfake generation, identity fraud). Access may be restricted to verified researchers or institutions.

\subsubsection{Future works.}

To the limitation on computational demand, we consider to explore more efficient network design, which guarantees the noise robustness while decrease the size of whole pipeline, aiming to faster inference and lower training burden.
Except for the pipeline optimization, we plan to explore the way of prior integration in talking head task. Specifically, the LQ talking head results of Wav2Lip model are seem as a kind of intermediate representation (like the face keypoints representation in GeneFace and GeneFace++). Following these works' idea, integrating such more semantic-rich intermediate representation into NeRF or 3DGS pipeline seems like a promising way to achieve better audio generality. Besides, we also notice that SOVQAE incurs decrease of lip-synchronization. Hence, alleviating this phenomena via more delicate network design is also our goal.
Besides, it is notable that SOVQAE still has mild damage in lip-synchronization metric to the output of Wav2Lip model, maybe we can explore an extra lightweight network to improve this.

\section{Additional Experiments}
\subsection{Plug-and-Play Experiment}
\label{plug-and-play}
To demonstrate the plug-and-play performance of our SOVQAE, we integrated SOVQAE into various base models. In this experiment, we follow our cascade pipeline workflow: 1) conduct face detection from output video from different base models; 2) resize the detected faces to $256\times 256$ then put them into SOVQAE; 3) resize back to detected shape and paste back to output video from base model. As shown in Tab.\ref{tab:plug-and-play}, SOVQAE improves visual quality metrics across different ID-specific base models, confirming its versatility and effectiveness. Notably, our cascade pipeline achieves the best overall performance by leveraging the exceptional audio-lip synchronization capabilities of Wav2Lip. This superior synchronization positively impacts PSNR and FID values, thereby enhancing overall video quality and ensuring cascade pipeline's top performance.

Additionally, we provide further visual evidence in Fig.~\ref{fig:app_add_exp}. Our SOVQAE significantly enhances high-fidelity details across different ID-specific base models, as evidenced by more detailed forehead wrinkles and finer tooth textures. These improvements highlight the effectiveness of SOVQAE in preserving fine-grained details while maintaining temporal consistency.

\begin{figure}[h]
    \centering
    \includegraphics[width=1.0\textwidth]{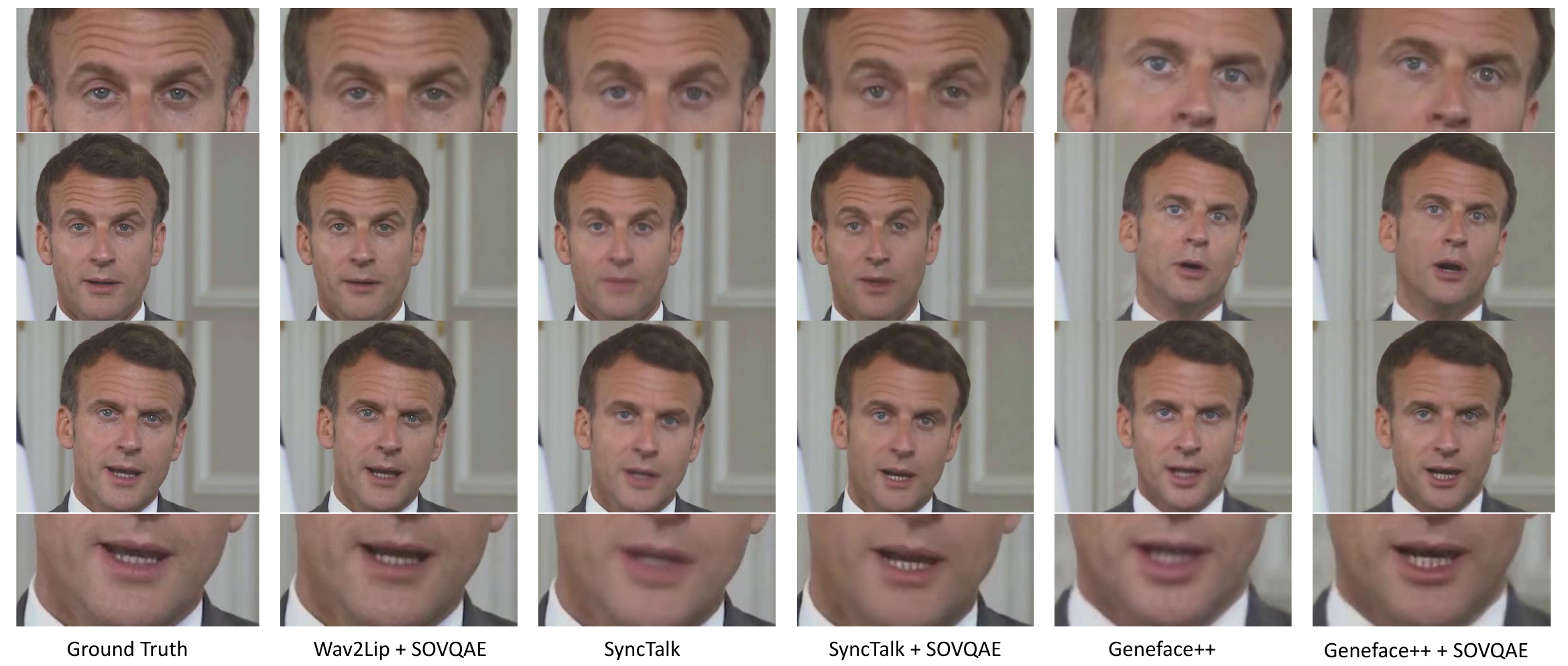}
    \caption{Visualization of the Plug-and-Play Experiment. From top to bottom, the figure shows: 
close-ups of the forehead in the first keyframe from different methods, the first keyframe, the second keyframe, and close-ups of the teeth in the second keyframe.}
    \label{fig:app_add_exp}
\end{figure}

\subsection{OOD Audio-driven Experiment}

As shown in Table~\ref{tab:OOD}, we present a detailed breakdown of each item in our out-of-distribution (OOD) audio-driven experiment. The evaluation covers 10 diverse audio clips spanning multiple languages—including Japanese, English, Chinese, Spanish, and French—as well as varying speaker ages and genders. Across all test cases, our pipeline consistently achieves the best lip synchronization performance.

\begin{table}[t]
\centering
\caption{There are 10 cross-lingual and cross-gender 20 second audios in OOD audio-driven experiment. We use different audios drive same subject and calculate LSE-C and LSE-D metrics.}
\setlength{\tabcolsep}{2pt} 
\resizebox{1.0\textwidth}{!}{\begin{tabular}{>{\centering\arraybackslash}m{2cm} p{1.5cm} p{1.5cm} p{1.5cm} p{1.5cm} p{1.5cm} p{1.5cm} p{1.5cm} p{1.5cm} p{1.5cm} p{1.5cm}}
\toprule
Method & \multicolumn{2}{c}{DINet\cite{zhang2023dinet}} & 
\multicolumn{2}{c}{IP-LAP\cite{zhong2023identity}} &\multicolumn{2}{c}{GeneFace++\cite{ye2023geneface++}} & \multicolumn{2}{c}{SyncTalk\cite{peng2023synctalk}} & \multicolumn{2}{c}{Ours} \\
\cmidrule(r){2-3} \cmidrule(r){4-5} \cmidrule(r){6-7} \cmidrule(r){8-9} \cmidrule(r){10-11}
Metrics &LSE-C$\uparrow$ & LSE-D$\downarrow$ & LSE-C$\uparrow$ & LSE-D$\downarrow$ & LSE-C$\uparrow$ & LSE-D$\downarrow$ & LSE-C$\uparrow$ & LSE-D$\downarrow$ & LSE-C$\uparrow$ & LSE-D$\downarrow$\\
\midrule
Audio 1 & 4.7768 &8.9890 &4.4024 &9.3161
  &4.8974	 &8.9521
 &4.2177&9.28129 & \textbf{6.2354} & \textbf{8.0987} \\
Audio 2 &4.3692 &9.0287 &4.1531 &8.6423
 &4.6195	 &8.7908
 &4.6328 &8.0043 & \textbf{5.9480} & \textbf{7.6968} \\
Audio 3 &5.6343 &8.6510 & 4.8950& 9.1588
&5.1553 &	8.8068
 &6.3866 &8.0077 & \textbf{7.4764} & \textbf{7.5445} \\
Audio 4 &4.2127 &8.5808 &4.8710 &8.6493
 &3.9984	 &8.2326
 &5.7213 &7.8372 & \textbf{6.7817} & \textbf{7.5449} \\
Audio 5 &\textbf{6.1828} &8.5279 &3.8994 &9.4514
 & 5.9873& 	8.5823
&4.4513 &8.8178 & {6.0246} & \textbf{7.9151} \\
Audio 6&6.2382 &7.9339 &5.1737 &8.5354
 &5.9934	 &7.7590
 &6.0401 &7.1978 & \textbf{7.1988} &\textbf{7.3960} \\
Audio 7 &4.6975 &8.7730 &4.6781 &9.2968
 &5.0774	&8.3832
  &5.4807 &8.2488 & \textbf{7.2357} & \textbf{7.1888} \\
Audio 8 &5.5323 &8.3505 &4.5502 &8.9811
 &5.9367	 &7.8090
 &5.4721 &7.9490 & \textbf{6.8894} & \textbf{7.2750} \\
Audio 9 &5.0672 &8.8984 &4.7241 &9.2790
 &5.5664	 &8.6923
 &5.7529 &7.7905 & \textbf{6.9524} & \textbf{7.5649} \\
Audio 10 &5.5989 &8.1889 & 2.4123& 9.2861
& \textbf{6.2121}& 	\textbf{7.9188}
&2.9075 &7.3401 & 3.7788 & 8.4436 \\
\textbf{average} &5.2310 &8.5922 &4.3759 &9.0596 &5.3445 &8.3927 &5.1063 &8.0474 & \textbf{6.4521} & \textbf{{7.6668}} \\
\bottomrule
\multicolumn{8}{l}{\textbf{Best}.}
\end{tabular}}
\label{tab:OOD}
\end{table}

\section{Proof of Theorem 3.1}
\label{append_3.1}
\subsection{Notation}

~~For all positive integer $n$, the set $[n]=\{1,2,\cdots, n\}$.

For all set $I$, the cardinality of $A$ is denoted by $|I$.

For all finite index sets $I$ and $J$, an $I\times J$ matrix $M$ over a ring $R$ is a $|I|\times |J|$ matrix whose rows is indexed by $I$ and whose columns is indexed by $J$, and for all $i\in I,j\in J$, the $(i,j)$-entry of $M_{ij}\in R$.

Let $M$ and $N$ be $I\times J$ and $J\times K$ matrices over $R$ respectively. The product of $MN$ is a $I\times J$ matrix over $R$ whose $(i,k)$-entry is
\[
  (MN)_{ik}=\sum_{j\in J}M_{ij}N_{jk}.
\]

For all $I\times  J$ matrix $M$ over $\RR$, the norm $\|M\|$ of $M$ is the Frobenius norm of matrix, that is,
\[
  \|M\|_F=\sqrt{\sum_{(i,j)\in I\times J}M^2_{ij}}.
\]

\subsection{One layer CNN}

We main consider the CNN for graphs, hence we always assume that the shape of input is $c_i\times h\times w$.

\begin{definition}
  A \emph{convolutional layer} $\lL$ is a data $(\ker\lL, p\in\NN^2, s\in\NN^2)$ which is consisted of
  \begin{enumerate}
    \item A tensor $\ker\lL$ of shape $c_i\times c_o\times k_h\times k_w$, whish is called the \emph{convolution kernel} of $\lL$. Where $c_i$ is the number of input channels and $c_o$ is the number of output channels.
    \item A pair $p=(p_h,p_w)$ is the \emph{padding} of $\lL$.
    \item A pair $s=(s_h,s_w)$ is the \emph{stride} of $\lL$.
  \end{enumerate}
  For convenience, we assume that $s_h$ and $s_w$ is a factor of $(h-k_h+p_h)$ and $(w-k_w+p_w)$ respectively. 
\end{definition}

For all convolutional layer $\lL=(\ker\lL, p, s)$, there is a map
\[
  F_\lL:\RR^{c_i\times h\times w} \rightarrow \RR^{c_o\times o_h\times o_w}
\]
corresponding to $\lL$, where $o_h=1+\frac{h-k_h+p_h}{s_h}$ and $o_w=1+\frac{w-k_w+p_w}{s_w}$. 

The domain $\RR^{c_i\times h\times w}$ should be considered as $c_i$ input channels of shape $h\times w$. Similarly, the codomain $\RR^{c_o\times h\times w}$ should be considered as $c_o$ output channels of shape $o_h\times o_w$.

Let $V$ and $W$ be $R$-modules. We denote $\Hom_{R}(V,W)$ as the space of $R$-linear maps from $V$ to $W$, which is also an $R$-module. And if 
\[
  V=\bigoplus_{i\in I}V_i,\quad\text{and}\quad W=\bigoplus_{j\in J}W_j
\]
then there is an $R$-isomorphism
\[
  \Hom_R\left(\bigoplus_{i\in I}V_i,\bigoplus_{j\in J}W_j\right)\cong \bigoplus_{(i,j)\in I\times J}\Hom_R(V_i,W_j).
\]
We refer to \cite[Part 3]{dummit2004abstract} for details.

The map $F_\lL$  is linear, then since
\[
  \RR^{c_i\times h\times w}=\bigoplus_{i=1}^{c_i}\RR^{h\times w}\quad\text{and}\quad\RR^{c_o\times o_h\times o_w}=\bigoplus_{j=1}^{c_o}\RR^{o_h\times o_w},
\]
there is a bijection
\[
  \Hom_\RR\left(\RR^{c_i\times h\times w},\RR^{c_o\times o_h\times o_w}\right)=\\
  \Hom_\RR\left(\bigoplus_{i=1}^{c_i}\RR^{h\times w},\bigoplus_{j=1}^{c_o}\RR^{o_h\times o_w}\right)\cong\bigoplus_{i,j}\Hom_\RR\left(\RR^{h\times w}, \RR^{o_h\times o_w}\right)
\]
the linear map $F_\lL$ can be identified as $c_o\times c_i$ many matrix over $\RR$ such that the $(i,j)$-matrix are linear maps 
\[
  F_{\lL,i,j}:\RR^{h\times w} \rightarrow \RR^{o_h\times o_w}.
\]

The linear map $F_{\lL,i,j}$ is the composition of
\[
  \RR^{h\times w} \xrightarrow{\iota} \RR^{(h+p_h)\times(w+h_w)} \xrightarrow{\hat{F}_{\lL,i,j}}\RR^{o_h\times o_w},
\]
where $\iota$ is a linear map represented by a $([h+p_h]\times[w+p_w])\times([h]\times [w])$ matrix whose $((a,b),(c,d))$-entry is
\[
  \iota_{(a,b),(c,d)}=\begin{cases}
    1, & \text{if } c=a+p_h,d=b+p_w \\
    0, & \text{otherwise}
  \end{cases}
\]
We can check that the map $\iota$ is an isometry embedding.

The linear map $\hat{F}_{\lL,i,j}$ can also be represented by $([o_h]\times [o_w])\times([h+p_h]\times [w+p_w])$ matrix such that the $((a, b),(c,d))$-entry is
\[
  F_{\lL,i,j,(a,b),(c,d)}=\begin{cases}
    (\ker\lL)_{i,j,x,y}, & \text{if } (\star) \\
    0, & \text{otherwise}
  \end{cases}
\]
where $(\ker\lL)_{i,j,x,y}$ is the $(i,j,x,y)$-entry in the tensor $\ker\lL$, and the condition
\[
  (\star):a=(c-1)k_h+x, b=(d-1)k_w+y, 1\le x\le k_h, 1\le y\le k_w.
\]

Now we use $\|M\|_\op$ to represent the operation norm of the matrix $M$. That is,
\[
  \|M\|_\op=\max_{\|x\|=1}\|Mx\|_2.
\]

By definition,
\[
  F_\lL=\hat{F}_\lL\circ\iota,
\] 
so we have 
\[
  \|F_\lL\|_\op\le \|\hat{F}_\lL\|_\op\|\iota\|_\op= \|\hat{F}_\lL\|_\op.
\]

Thus, to estimate the operation norm $\|F_\lL\|$ of $F_\lL$, it suffices to find the operation norm of $\|\hat{F}_\lL\|$. Now we first prove a simple but useful lemma.

\begin{lemma}
  Let $\aA$ be an $M\times N$ matrix that is be identity with an $m\times n$ block matrix $(A_{ij})_{ij}$, then
  \[
    \|\aA\|_\op \le \sqrt{mn}\max_{i,j}\|A_{ij}\|_\op.
  \]
\end{lemma}

\begin{proof}
  Consider the easier case when $m=1$, for all $x\in\RR^N$, we split it as $(x_j)_{1\le j\le n}$ which is compatible with the block matrix representation of $\aA$. We represent the vector $x_i$ by $(x_{jk})_k$. Then
  \[
    \|x\|_2=\sqrt{\sum_{j=1}^n \sum_kx_{jk}^2}=\sqrt{\sum_{j=1}\|x_j\|_2^2}.
  \]

  Now $\aA x=\sum_{j=1}^n A_{1j}x_j$, so
  \[
    \begin{aligned}
      \|\aA x\|_2=\left\|\sum_{j=1}^n A_{1j}x_j\right\|_2 & \le\sum_{j=1}^n\|A_{1j}x_j\|_2 \\
      & \le\sum_{j=1}^n\|A_{1j}\|_\op\|x_j\|_2 \\
      & \le \sqrt{\sum_{j=1}^n \|A_{1j}\|_\op^2}\sqrt{\sum_{j=1}^n\|x_j\|_2} \\
      & \le \sqrt{n}\max_{1j}\|A_{1j}\|_\op\|x\|_2,
    \end{aligned}
  \]
  so $\|\aA\|_\op\le \sqrt{n}\max_{1j}\|A_{1j}\|_\op$.

  Then when $n=1$, for all $x\in\RR^N$, $\aA x=(A_{i1}x)_{1\le i\le m}$, so by above
  \[
    \|\aA x\|_2=\|(A_{i1}x)^T\|_2=\sqrt{\sum_{i=1}^m\|A_{i1}x\|_2^2}\le\sqrt{\sum_{i=1}^m\|A_{i1}\|_\op^2\|x\|_2^2}\le\sqrt{m}\max_{i}\|A_{i1}\|_\op\|x\|_2,
  \]
  so $\|\aA\|_\op\le \sqrt{n}\max_{1j}\|A_{i1}\|_\op$.

  Now for the general case, denote the $1\times n$ block matrix $\aA_{i}=(A_{ij})_{1\le j\le n}$, then $\aA$ can be identity with the $m\times 1$ block matrix $(\aA_i)_{1\le i\le m}$, so by above
  \[
    \|\aA\|_\op\le\sqrt{m}\max_{i}\|\aA_i\|_\op\le\sqrt{m}\max_{i}\left(\sqrt{n}\max_{j}\|A_{ij}\|_\op\right)=\sqrt{mn}\max_{i,j}\|A_{ij}\|_\op,
  \]
  hence, $\|\aA\|_\op \le \sqrt{mn}\max_{i,j}\|A_{ij}\|_\op$.
\end{proof}

By the above Lemma, we have
\[
  \|\hat{F}_\lL\|_\op\le \sqrt{c_ic_o}\max_{i,j}\|\hat{F}_{\lL,i,j}\|_\op.
\]

Now we will estimate the operation norm of $\hat{F}_{\lL,i,j}$. 

By the definition of $\hat{F}_{\lL,i,j}$, we have a key observation: we can rearrange the index such that the rows of $\hat{F}_{\lL,i,j}$ are very similar, the matrix after rearranging is denoted by $\Tilde{F}_{\lL,i,j}$. Moreover, we have
\[
  \|\hat{F}_{\lL,i,j}\|_\op\le\|\Tilde{F}_{\lL,i,j}\|_\op.
\]

For every vector $v=(v_i)_i\in\RR^n$ and for all positive integer $m$, we define the shifted vector
\[ 
  v[m]=(\underbrace{0,0,\cdots,0}_{m\text{ times}}, v_1, v_2, \cdots, v_{n-m}).
\]
After a proper rearrangement of index, let $\wp_k$ be the $k$-th row of $\Tilde{F}_{\lL,i,j}$, by calculation,
\[
  \wp_k = \wp_1[s_h(k-1)],
\]
hence, we denote $\wp_1$ by $\wp$. By the fact
\[
  \|A\|_\op = \sqrt{\rho(AA^T)},
\]
where $\rho(A)$ is the spectral radius of $A$. See details for \cite[Theorem 6.15]{shores2007applied}.

Notice that $\Tilde{F}_{\lL,i,j}\Tilde{F}_{\lL,i,j}^T$ is a positive definite symmetric Toeplitz matrix
  \[
    \begin{pmatrix}
      c_0 & c_{1} & c_2 &  \cdots & c_n \\
      c_1 & c_0 & c_1 & \cdots & c_{n-1} \\
      c_2 & c_1 & c_0 & \cdots & c_{n-2} \\
      \vdots & \vdots & \vdots & \cdots & \vdots \\
      c_n & c_{n-1} & c_{n-2} & \cdots & c_0 
    \end{pmatrix}
  \]
  where $c_k=\wp_1\cdot \wp_{k+1}, n=o_ho_w-1$.

\begin{proposition}
  If $s_h\ge k_h$ and $s_w\ge k_w$, then $\|\hat{F}_{\lL,i,j}\|_\op\le\|\wp\|_2=\|\ker\lL_{ij}\|_F$.
\end{proposition}

\begin{proof}
  By calculation, if $s_h\ge k_h$ and $s_w\ge k_w$, then $\wp_u\cdot \wp_v=0$ for all $u\ne v$, hence $AA^T=\diag(\|\wp_1\|^2,\|\wp_2\|^2,\cdots)$, and by definition of $F_\lL$, 
  \[
    \|\wp_k\|_2=\|\wp\|_2=\|\ker\lL_{ij}\|_F,
  \]
  hence, 
  \[
    \|\hat{F}_{\lL,i,j}\|_\op\le\|\Tilde{F}_{\lL,i,j}\|_\op=\sqrt{\rho\left(\Tilde{F}_{\lL,i,j}\Tilde{F}_{\lL,i,j}^T\right)}=\sqrt{\|\wp\|_2^2}=\|\wp\|_2=\|\ker\lL_{ij}\|_F,
  \]
  we are done.
  \end{proof}

  In practical applications, often $k_h\ll h,k_w\ll w$, which means that $AA^T$ is a positive definite symmetric banded Toeplitz matrix, there are many theorems estimating the operation norm of such matrices. Now we state some relative results.

  Let $f$ be a real value function. The \emph{essential supremum} $M_f$ of $f$ is the smallest number such that $f(x)\le M_f$ for all $x$ except on a set of measure 0. 

  For all Toeplitz matrix
  \[
    C=\begin{pmatrix}
      c_0 & c_{-1} & c_{-2} &  \cdots & c_{-n} \\
      c_1 & c_0 & c_{-1} & \cdots & c_{-(n-1)} \\
      c_2 & c_1 & c_0 & \cdots & c_{-(n-2)} \\
      \vdots & \vdots & \vdots & \cdots & \vdots \\
      c_n & c_{n-1} & c_{n-2} & \cdots & c_0 
    \end{pmatrix}
  \]
  the Fourier series $f_C(\lambda)$ associate to $C$ is defined by
  \[
    f_C(\lambda)=\sum_{k=-n}^nc_k\mathrm{e}^{\mathrm{i}k\lambda},
  \]
  then we have
  \[
    \rho(C)\le 2M_{|f_C|}.
  \]
   In particular, if $C$ is Hermitian, that is, $c_{-k}=c_k^*$, then
  \[
   \rho(C)\le M_{|f_C|}.
  \]

  We refer to \cite[Lemma 4.1]{gray2006toeplitz} for details.

  We denote the Fourier series $f_{\Tilde{F}_{\lL,i,j}\Tilde{F}_{\lL,i,j}^T}$ associate to $\Tilde{F}_{\lL,i,j}\Tilde{F}_{\lL,i,j}^T$ by $f_{\lL,i,j}$. The above result gives that
  \[
    \|\hat{F}_{\lL,i,j}\|_\op \le\|\Tilde{F}_{\lL,i,j}\|_\op= \sqrt{\rho\left(\Tilde{F}_{\lL,i,j}\Tilde{F}_{\lL,i,j}^T\right)} \le \sqrt{M_{\left|f_{\lL,i,j}\right|}}.
  \]

  Therefore, we concliude that 
  \[
    \|F_{\lL}\|_\op\le \max_{i,j}\sqrt{c_ic_oM_{\left|f_{\lL,i,j}\right|}}.
  \]
  
  As a conclusion, if $\lL$ is a convolutional layer, and $F_\lL$ is the map correspoeing to $\lL$, then for all $x,y\in\RR^{c_i\times h\times w}$, we have
  \[
    \|F_\lL(x)-F_\lL(y)\|_F\le M_\lL\|x-y\|_F.
  \]
  where $M_\lL=\max_{i,j}\sqrt{c_ic_oM_{\left|f_{\lL,i,j}\right|}}$.

  We always assume that the activation function is Lipschitz continuous and its Lipschitz constant is known.
  
  \subsection{Multiple layers CNN}
  
  The above results show that single-layer CNN is Lipschitzian continuous with computable Lipschitz constant. Since any CNN network is a composite of multiple single-layer CNN networks and activation functions, we conclude that any CNN is Lipschitzian continuous with computable Lipschitz constant.

  We, therefore, reach the following conclusion.

  Let $\epsilon:\RR^{c_i\times h\times w} \rightarrow \RR^{c_o\times o_h\times o_w}$ is a multi-layer CNN, $\lL_k$ the $k$-th convolutional layer, and let the Lipschitz constant of the $k$-th activation functions be $L_{k}$. Then
  \[
    \forall x,y\in\RR^{c_i\times h\times w}, \quad \|\epsilon(x)-\epsilon(y)\|_F\le L_\epsilon\|x-y\|_F,
  \]
  where $L_\epsilon=\prod_{k}M_{\lL_k}L_k$.$\qedsymbol$

\section{Proof of Theorem 3.2}
\label{append_3.2}
Now we mainly consider a CNN as a Lipschitzian continuous function $\epsilon$ with Lipschitzian constant $L_{\epsilon}$. 
The spaces $\sS$ of graphs of size $h\times w$ is a subset of the space $\RR^{c_i\times h\times w}$. The range $\epsilon(\sS)$ can be seen as an encoding of $\sS$. 
In this paper, we utilize VQ-AE to learn suitable encoding of space $\sS$ and simultaneously select anchors in this encoding. Our goal is for this network to learn the features of high-quality images and minimize the distance between codebook and the features. Meanwhile, we also wish the network to be robust with mild input disturbance .Next, we will demonstrate the feasibility of our approach. 
  
We first prove it in the case of \textbf{single channel}:

Let $\{ \epsilon, \delta, {\bf C} \subset \mathbb{R}^c, g_{{\bf C}}\}$ be VQ-AE with single channel latent space, where $\epsilon$ is the CNN encoder with Lipschitzian constant $L_{\epsilon}$, and $\mathbf{C}\subset \epsilon(\sS)$ is a set of codebook anchors. Define $d_{\mathbf{C}}=\min\{\|a-b\|_F:a,b\in \mathbf{C},a\ne b\}$ and $\gamma$ be the maximal distance of high-quality image latent to closest anchor, that is, $\gamma=\max_{p\in \epsilon(\mathbf{HQ})}d(p,\mathbf{C})$, where $\mathbf{HQ}$ is the space of high quality images, and $p \in \mathbb{R}^c$ is single channel vector. Assume that $2\gamma< d_{\mathbf{C}}$ (for almost all cases, it is naturally satisfied). For all low-quality image $y$ corresponding to a high-quality image $x$ such that
  \[
    \|x-y\|_F<\frac{d_{\mathbf{C}}-2\gamma}{2L_\epsilon},
  \]
  then we have 
  \[
    \|\epsilon(x)-\epsilon(y)\|_F\le L_\epsilon\|x-y\|_F<\frac{d_{\mathbf{C}}-2\gamma}{2}.
  \]
  
  Let $g_{\mathbf{C}}(\epsilon(x))=s$, that is, $s \in \textbf{C}$ is the anchor closest to $\epsilon(x)$. Then
  \[
    \|\epsilon(x)-s\|=d(\epsilon(x),\mathbf{C})\le\max_{p\in \epsilon(\mathbf{HQ})}d(p,\mathbf{C})=\gamma.
  \]
  
  For anchor $s$, 
  \[
    \begin{aligned}
        \|\epsilon(y)-s\|_F & =\|\epsilon(y)-\epsilon(x)+\epsilon(x)-s\|_F \\
        & \le \|\epsilon(y)-\epsilon(x)\|_F +\|\epsilon(x)-s\|_F \\
        & < \frac{d_{\mathbf{C}}-2\gamma}{2}+\gamma \\
        & = \frac{d_{\mathbf{C}}}{2}
    \end{aligned}
  \]
  so $\|\epsilon(y)-s\|_F<\frac{d_{\mathbf{C}}}{2}$.

  For each anchor $s\ne a\in \mathbf{C}$, we claim that $\|\epsilon(y)-a\|\ge \frac{d_{\mathbf{C}}}{2}$: if not, then
  \[
    \begin{aligned}
        d_{\mathbf{C}} & =\min\{\|a-b\|_F:a,b\in \mathbf{C},a\ne b\} \\
        & \le \|a-s\|_F \\
        & =\|a-\epsilon(x)+\epsilon(x)-s\|_F \\
        & \le \|\epsilon(x)-a\|_F+\|\epsilon(x)-s\|_F \\
        & < \frac{d_{\mathbf{C}}}{2} + \frac{d_{\mathbf{C}}}{2} \\
        & = d_{\mathbf{C}}
    \end{aligned}
  \]
  so $d_{\mathbf{C}}<d_{\mathbf{C}}$, a contradiction. Thus, $s$ is also the anchor closest to $\epsilon(y)$, which means that
  \[
    g_{\mathbf{C}}(\epsilon(y))=s=g_{\mathbf{C}}(\epsilon(x)).
  \]

  When latent space is \textbf{multi-channel} ($\epsilon(x)\subset \mathbb{R}^{h\times w\times
  c}$), for all low-quality image $y$ corresponding to a high-quality image $x$ such that
  \[
    \|x-y\|_F<\frac{d_{\mathbf{C}}-2\gamma}{2L_\epsilon}.
  \]
  
  Let $\epsilon(x)_{ij}, \epsilon(y)_{ij} \in \mathbb{R}^c$ be two single-channel vector of $\epsilon(x), \epsilon(y)$, and $s_{ij} = g_{\mathbf{C}}(\epsilon(x)_{ij}) \in \textbf{C}$ be the closest anchor of $\epsilon(x)_{ij}$. Then
  \[
    \|\epsilon(x)_{ij}-\epsilon(y)_{ij}\|_F \le \|\epsilon(x)-\epsilon(y)\|_F<\frac{d_{\mathbf{C}}-2\gamma}{2},
  \]
  by the proof of the single channel case, 
  \[
    g_{\mathbf{C}}(\epsilon(y)_{ij})=g_{\mathbf{C}}(\epsilon(x)_{ij}),
  \]
  which shows that
  \[
    g_{\mathbf{C}}(\epsilon(y))=(g_{\mathbf{C}}(\epsilon(y)_{ij}))_{ij}=(g_{\mathbf{C}}(\epsilon(x)_{ij})_{ij}=g_{\mathbf{C}}(\epsilon(x)).
  \]
  Therefore, we could get the desired high-quality image $x$ corresponding to $y$ after decoding by $\delta$. $\qedsymbol$

  \begin{remark}
      If the VQ-AE $\{ \epsilon, \delta, {\bf C} \subset \mathbb{R}^c, g_{{\bf C}}\}$ is convergent, the distance of high-quality images to codebook $\mathbf{C}$ is extremely small, which means that $\gamma\ll 1$ is negligible.
  \end{remark}
  
  As a corollary, if $\mathbf{I}_{high}$ is a high-quality image, and $\mathbf{I}_{up}=\mathbf{I}_{high}+\nN$ where $\nN$ is the Gaussian degradation with $\|\nN\|_F<\frac{d_{\mathbf{C}}-2\gamma}{2L_\epsilon}$, then the above shows that
  \[
    g_{\mathbf{C}}\left(\epsilon(\mathbf{I}_{up})\right)=g_{\mathbf{C}}\left(\epsilon(\mathbf{I}_{high})\right).
  \]
  
  Therefore, for CNN $\epsilon$, it can correctly match images with a distance of no more than $\frac{d_{\mathbf{C}}-2\gamma}{2L_{\epsilon}}$ from the high-quality images. The correctness of this result requires us to ensure that the selected anchors are all high-quality images, which requires us to use high-quality images for training so that the model can extract features from high-definition images. The judgment range of this model is determined by three parts: first, the Lipschitz constant $L_{\epsilon}$ of $\epsilon$, and second the distance $c$ between anchors in $\epsilon(\sS)$, and third, the distribution of anchors in the original space $\sS$. These three factors are interdependent. To be precise, while ensuring that the anchor points taken are all high-quality images, the effective range of the model is
  \[
    \bigcup_{a\in \mathbf{C}}B\left(\epsilon^{-1}(a),\frac{d_{\mathbf{C}}-2\gamma}{2L_{\epsilon}}\right),
  \]
  where $B(a,r)=\{x:\|x-a|<r\}$. 

  This requires us to ensure that the anchors are all high-quality images, and to make the anchors as uniform as possible in the original space $\sS$ if the number of anchors are fixed, while also making the ratio $\frac{d_{\mathbf{C}}-2\gamma}{2L_{\epsilon}}$ as large as possible. It should be noted that the values of $d_{\mathbf{C}}$ and $L_{\epsilon}$ are correlated. For example, we can always multiply by a constant $\alpha<1$ to change the Lipschitz constant to $\alpha L_{\epsilon}$, but at the same time, the parameter $d_{\mathbf{C}}$ and $\gamma$ also becomes $\alpha d_{\mathbf{C}}$ and $\alpha\gamma$ respectively, which implies that the ratio remains unchanged. 

  In theory, the selection of anchors is independent of CNN $\epsilon$. However, it's worth noting that as $\epsilon$ changes, the model's ability to capture features of high-quality images also changes, affecting the selection of anchor points. Having a large distance $d_{\mathbf{C}}$ between anchor points is not ideal, as it affects both the Lipschitz constant of the CNN, as mentioned above, and the distribution of anchors in the space $\sS$.

  Similarly, we observe that having a Lipschitz constant $L_\epsilon$ that is too small is also not ideal, as it may cause the distances between images of different content to be too small, leading to decreased robustness of the model. An ideal scenario is where the distances between images of the same content are compressed while the distances between images of different content are relatively large. Therefore, we choose to impose reasonable requirements on the distribution of anchor points to train and improve the effectiveness of our model.

  Due to the characteristics of images, in practical applications, we often process each part of the image locally. Therefore, we apply the above method to each part to achieve more accurate results, and each part can be regarded as a whole. Therefore, our above argument still holds in this case.

\section{More implementation details}
\label{append_more_detail}
Each network component is displayed in Table \ref{table4}. In experiment, we have 4 downsample blocks. hence we have $m=4, f=16$ in table. 
\begin{table}[h]
\centering
\begin{tabular}{p{0.45\textwidth}|p{0.45\textwidth}}
\toprule
\multicolumn{1}{c|}{\textbf{Encoder}} & \multicolumn{1}{c}{\textbf{Decoder}} \\
\midrule
$x \in \mathbb{R}^{H \times W \times C}$ & $z_q \in \mathbb{R}^{h \times w \times n_z}$ \\
Conv2D $\to \mathbb{R}^{H \times W \times C'}$ & Conv2D $\to \mathbb{R}^{h \times w \times C''}$ \\
$m \times \{ \text{Residual Block, Downsample Block} \} \to \mathbb{R}^{h \times w \times C''}$ & Residual Block $\to \mathbb{R}^{h \times w \times C''}$ \\
Residual Block $\to \mathbb{R}^{h \times w \times C''}$ & Non-Local Block $\to \mathbb{R}^{h \times w \times C''}$ \\
Non-Local Block $\to \mathbb{R}^{h \times w \times C''}$ & Residual Block $\to \mathbb{R}^{h \times w \times C''}$ \\
Residual Block $\to \mathbb{R}^{h \times w \times C''}$ & $m \times \{ \text{Residual Block, Upsample Block} \} \to \mathbb{R}^{H \times W \times C'}$ \\
GroupNorm, Swish, Conv2D $\to \mathbb{R}^{h \times w \times n_z}$ & GroupNorm, Swish, Conv2D $\to \mathbb{R}^{H \times W \times C}$ \\
\bottomrule
\end{tabular}
\caption{High-level architecture of the encoder and decoder of our SOVQAE. The design of the networks follows the architecture presented in \cite{isola2017image} with no skip-connections. For the discriminator, we use a patch-based model. Note that $h = \frac{H}{2^m}, w = \frac{W}{2^m}$ and $f = 2^m$.}
\label{table4}
\end{table}

\end{document}